\documentclass{article}

\usepackage{microtype}
\usepackage{graphicx}
\usepackage{subcaption}
\usepackage{booktabs}
\usepackage{hyperref}
\usepackage{amsmath}
\usepackage{amssymb}
\usepackage{mathtools}
\usepackage{amsthm}

\theoremstyle{plain}
\newtheorem{theorem}{Theorem}[section]

\newtheorem{proposition}[theorem]{Proposition}

\theoremstyle{definition}

\theoremstyle{remark}

\usepackage{enumitem}
\newcommand{\ket}[1]{\left|#1\right\rangle}
\usepackage[most]{tcolorbox}
\usepackage{xcolor}

\newcounter{examplebox}

\definecolor{mycolor}{RGB}{20,0,145}

\newtcolorbox{mycolorbox}[1][]{
  colframe=mycolor,
  colback=mycolor!4!white,
  title=#1
}

\newtcbtheorem{SidebarDef}{Definition}{
    enhanced,
    frame hidden,                   % No box frame
    colback=gray!3!white,           % Almost transparent background
    overlay={                       % The "Sidebar" (Left Line)
        \draw[line width=1.5pt, black!70!white] (frame.north west) -- (frame.south west);
    },
    sharp corners,
    top=2pt, bottom=2pt, left=6pt, right=2pt,
    fonttitle=\bfseries\small\scshape, 
    coltitle=black,
    attach title to upper={\ :\ },  
    colbacktitle=gray!3!white,      
    description font=\mdseries\itshape
}{def}

\newtcbtheorem{SidebarProp}{Background}{
    enhanced,
    frame hidden,
    colback=blue!2!white,           % Very subtle blue tint
    overlay={
        \draw[line width=1.5pt, blue!40!black] (frame.north west) -- (frame.south west);
    },
    sharp corners,
    top=2pt, bottom=2pt, left=6pt, right=2pt,
    fonttitle=\bfseries\small\scshape,
    coltitle=blue!40!black,
    attach title to upper={\ :\ },
    colbacktitle=blue!2!white,
    description font=\mdseries\itshape
}{prop}

\usepackage[capitalize,noabbrev]{cleveref}
\crefalias{examplebox}{example}
\Crefname{examplebox}{Example}{Examples}

\usepackage{listings}
\usepackage[accepted]{icml2026}

\icmltitlerunning{Quantum Program Generation Prioritizes Validity}

\begin{document}

\twocolumn[
  \icmltitle{Position: Quantum Program Generation Must Prioritize Validity \texorpdfstring{\\}{} Over Probabilistic Scaling}


  \icmlsetsymbol{equal}{*}

  \begin{icmlauthorlist}
    \icmlauthor{Junhao Song}{equal,doc}
    \icmlauthor{Yu Zhou}{equal,ese}
    \icmlauthor{William Knottenbelt}{doc}
    \icmlauthor{Yudong Cao}{zpta,bcg}
  \end{icmlauthorlist}

  \icmlaffiliation{doc}{Department of Computing, Imperial College London, London, United Kingdom}
  \icmlaffiliation{ese}{Department of Earth Science and Engineering, Imperial College London, London, United Kingdom}
  \icmlaffiliation{zpta}{Zapata Quantum, Boston, Massachusetts, United States of America}
  \icmlaffiliation{bcg}{BCG X AI Science Institute, Boston Consulting Group, Boston, United States of America}
  
  \icmlcorrespondingauthor{William Knottenbelt}{w.knottenbelt@imperial.ac.uk}
  \icmlcorrespondingauthor{Yudong Cao}{ycao@zapataquantum.com}
  
  \icmlkeywords{Quantum computing, generative models, formal verification, scaling laws, probabilistic methods}

  \vskip 0.3in
]


\printAffiliationsAndNotice{\icmlEqualContribution}  % no special notice (required even if empty)

\begin{abstract}
The scaling hypothesis assumes that increasing model parameters yields emergent reasoning capabilities. This position paper argues that applying this probabilistic paradigm to generic quantum circuit synthesis is a directional error. Unlike natural languages, quantum circuits require strict adherence to mathematical constraints that manifest a significant syntax-semantics gap. Training on unverified quantum programs means that models learn syntax but fail to capture the physical semantics of the Hilbert space. Since the valid subset of circuit designs decays exponentially with the number of qubits, post-hoc filtering is mathematically intractable. We propose a pivot from human-centric copilots to verifier-centric agents. We integrate hierarchical constraints, topological masks, and symbolic proxies directly into generation. Our analysis suggests that scale alone cannot bridge the validity gap. Verification-aware architectures offer a viable path for modular quantum program generation. These considerations point toward generation methods that encode task-specific rules of quantum information, rather than relying on imitation alone.
\end{abstract}

%%%%%%%%%%%%%%%%%%%%%%%%%%%%%%%%%%%%%%%%%%%%%%%%%%%%%%%%%%%%%%%%%%%%%%%%%%%%%%%
%% MAIN CONTENT - Modular Sections
%%%%%%%%%%%%%%%%%%%%%%%%%%%%%%%%%%%%%%%%%%%%%%%%%%%%%%%%%%%%%%%%%%%%%%%%%%%%%%%

%%%%%%%%%%%%%%%%%%%%%%%%%%%%%%%%%%%%%%%%%%%%%%%%%%%%%%%%%%%%%%%%%%%%%%%%%%%%
%% Section 1: Introduction
%%%%%%%%%%%%%%%%%%%%%%%%%%%%%%%%%%%%%%%%%%%%%%%%%%%%%%%%%%%%%%%%%%%%%%%%%%%%

\section{Introduction}
\label{sec:intro}

Machine learning research has flourished often with the conflation of fluency with validity in sequence generation tasks. A syntactically imperfect function generated by large language models (LLMs) often retains semantic utility \cite{austin2021program}. Conversely, a syntactically correct statement may be semantically wrong, but may still be an informative data point for model training \cite{chen2021evaluating}. In classical software engineering, this conflation is forgiving. An off-by-one error in a Python script is a localized failure; it is often debuggable via simple runtime execution. For quantum computing, however, it is dangerous to inherit this ``copilot'' tradition. It blindly imports the assumption that partial correctness yields partial utility. In the quantum realm, a misplaced operation is not merely a bug. Depending on its position, it can invalidate the entire circuit design. A single misplaced gate does not just introduce a localized error. It can destroy the global interference pattern required for computation. Figure~\ref{fig:teaser} contrasts this naive closed-loop paradigm with a verifier-centric alternative, clarifying that the core failure lies not in the absence of verification, but in its post-hoc and non-constructive use.

\begin{figure}[htbp]
    \centering
    \includegraphics[width=\linewidth]{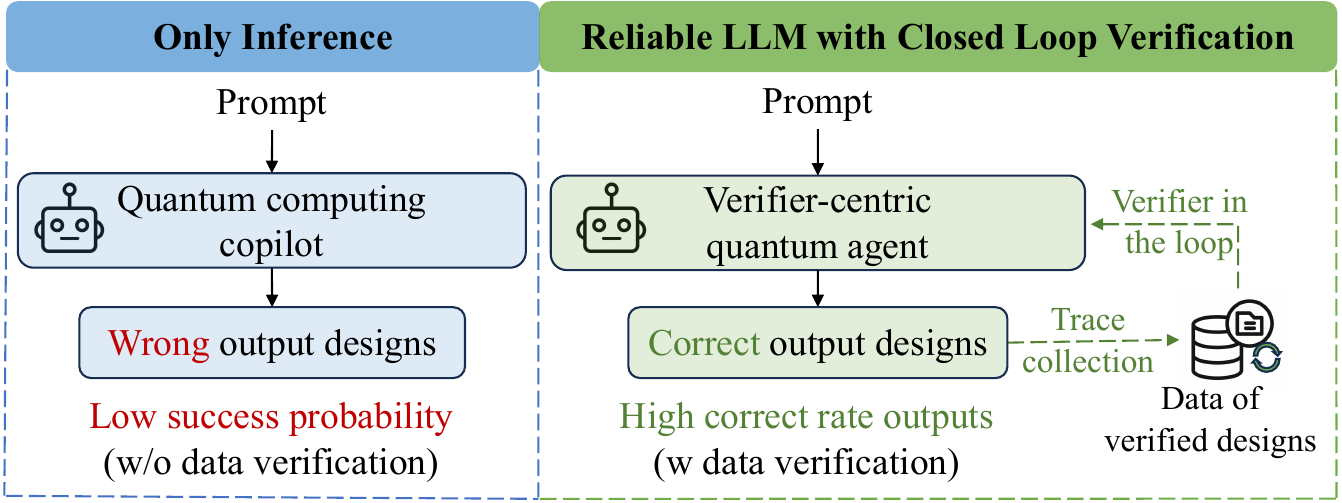}
    \caption{\textbf{Comparison of the traditional method and reliable LLM with closed loop verification.}
    When data verification is integrated into the training loop, only validated designs from the constrained validation space are used, ensuring higher correct rate outputs.}
    \label{fig:teaser}
\end{figure}

Recent systems like IBM's Qiskit Code Assistant \cite{dupuis2024qiskit} and DeepMind's AlphaTensor-Quantum \cite{ruiz2025quantum} show promise. However, a critical dissection reveals their success stems from domain-specific verification loops \cite{romeraparedes2024mathematical} or curated human-verified datasets \cite{vishwakarma2024qiskit}. It does not stem from raw parameter scaling. Outside these verified islands, the ``scaling hypothesis'' \cite{kaplan2020scaling} faces a hard mathematical wall. Standard LLMs maximize the likelihood of the next token. They do not maximize the fidelity of the resulting quantum state. There is a fundamental misalignment between the probabilistic objective of the model and the deterministic constraints of the Hilbert space.

The magnitude of this misalignment is often underestimated. The unitary group $\mathrm{SU}(2^n)$ is a continuous manifold of dimension $4^n - 1$ \citep{nielsen2010quantum}, whereas any polynomial-depth circuit ansatz spans only an $O(\mathrm{poly}(n))$-dimensional submanifold. As a consequence, the set of circuit designs implementing a given target unitary occupies an exponentially sparse subset of the space of all $n$-qubit circuits \cite{mele2024introduction}, a fact we formalize via Haar-measure concentration in Section \ref{subsec:gap}. For generic synthesis tasks, valid outcomes therefore constitute a vanishingly small fraction of the search space~\citep{bouland2018quantum}. In this regime, scaling computation does not bridge the entropy gap. It merely allows the model to hallucinate more convincingly. A larger model simply learns to generate more sophisticated noise.

% \vspace{-0.1em}
\begin{tcolorbox}[
    colback=blue!4!violet!3!white,
    colframe=blue!30!violet!50,
    boxrule=0.8pt,
    arc=2mm,
    left=8pt, right=8pt, top=8pt, bottom=8pt
]   
    %We characterize the application of probabilistic scaling to domains with hard constraints such as quantum computing as a category error. Generative models for such domains must stop simulating the syntax of scientific code and start internalizing the semantics of real-world physical laws with verified data.
    We argue that quantum program generation and other similar scientific domains with hard physical and mathematical constraints expose a fundamental syntax–semantics gap, where applying probabilistic scaling of large models can be entirely counterproductive. 
    The exponential sparsity of functionally correct outputs makes post-selection intractable, so reliability must come from a constructively verified system.
\end{tcolorbox}
% \vspace{-0.1em}

To transition from probabilistic mimicry to scientific discovery, we formalize three critical shifts:

\begin{enumerate}[leftmargin=*, noitemsep, topsep=-2pt]
    \item \textbf{Data poisoning (Section \ref{sec:data_crisis}):} Training on unverified open-source code introduces a structural bias against physical validity, or data poisoning in the broadest sense of a corrupted training signal. Public quantum code carries no correctness guarantee: parseable code need not implement a valid circuit. Models trained on this distribution learn syntax but fail to capture Hilbert space semantics. We hypothesize that this leads to \textbf{inverse scaling} \cite{mckenzie2023inverse}. Larger models more faithfully imitate the biased distribution, becoming more confident in physically invalid outputs.
    
    \item \textbf{The complexity trap (Section \ref{sec:complexity}):} We demonstrate that the sparsity of valid circuits renders post-hoc filtering computationally intractable for generic synthesis. Beyond approximately 40-50 qubits, classical verification becomes computationally intractable \cite{dalzell2020howmany,arute2019quantum}. This renders post-hoc filtering infeasible. The exponential decay of the valid subspace creates a coupled barrier that scaling model parameters alone cannot overcome.
    
    \item \textbf{Verifier-centric agents (Section \ref{sec:verifier}):} Human-in-the-loop verification is ill-suited for this domain. Human verification is cognitively impractical in the regime of a large number of qubits. We propose replacing human supervisors with formal solvers. We introduce constructive verification protocols. These integrate hierarchical constraints (topological masks, symbolic heuristics, and modular simulation) along with trace-based corpora to enable validity enforcement directly during generation.
\end{enumerate}

%%%%%%%%%%%%%%%%%%%%%%%%%%%%%%%%%%%%%%%%%%%%%%%%%%%%%%%%%%%%%%%%%%%%%%%%%%%%
%% Section 2: The Data Crisis
%%%%%%%%%%%%%%%%%%%%%%%%%%%%%%%%%%%%%%%%%%%%%%%%%%%%%%%%%%%%%%%%%%%%%%%%%%%%

\section{The Structural Limitations of Probabilistic Scaling}
\label{sec:data_crisis}

The prevailing scaling hypothesis implicitly assumes that increasing model parameters yields emergent reasoning capabilities \cite{kaplan2020scaling,hoffmann2022training}. While empirically validated in loosely constrained domains like natural language, this paradigm encounters fundamental difficulties in {quantum program generation}. We argue that applying probabilistic scaling to arbitrary quantum circuit design constitutes a structural mismatch that cannot be resolved by scaling alone.

% \begin{figure}[htbp]
%     \centering
%     \includegraphics[width=0.95\linewidth]{figs/The_Gap_between_Syntax_and_Semantics.pdf}
%     \caption{From syntactic fluency to physical invalidity: A category error in quantum program generation.}
%     \label{fig:category-error-quantum}
% \end{figure}

\subsection{The Syntax-semantics Gap}\label{subsec:gap}

In classical programming, the syntax–semantics gap is often narrow in practice because semantic correctness can be cheaply checked through execution, tests, debuggers and introspection, while syntax errors are frequently localized and repairable (Figure \ref{fig:fig2_comparison}). For instance, it has been shown that automated recovery can repair the majority of real-world syntactically invalid programs within seconds \cite{diekmann2020dontpanic}, which is evidence that many ``syntactically imperfect'' programs still preserve enough structural intent to be repaired into a meaningful artifact. We illustrate this fundamental difference with a detailed example in \Cref{ex:fallacy} of the Appendix.

Recent AI4Science results indicate that when outputs must satisfy hard mathematical constraints, LLMs behave primarily as high-recall proposal generators rather than reliable constrained designers: their outputs often appear plausible yet fail formal checks, and performance improves mainly when an external verifier or solver is integrated into the loop. In planning, it is shown that autoregressive LLMs are unreliable self-verifiers, motivating LLM-Modulo architectures where model-based verifiers enforce constraints and guide iterative repair \cite{kambhampati2024llms}. In formal mathematics, evaluations of Lean4 autoformalization reveal persistent failures on harder theorems \cite{gulati2024evaluation}, while APOLLO demonstrates that substantial gains arise from compiler-guided proof repair rather than raw generation \cite{ospanov2025apollo}. Together, these findings support the quantum-relevant claim that LLMs struggle in mathematically constrained domains unless generation is tightly coupled with explicit checking and repair.

A second, independent limitation arises in domains where the true optimization objective only becomes visible after crossing abstraction layers i.e., leaky abstraction. In such settings, local high-level proxies are unreliable because downstream compilation, synthesis, and optimization introduce nonlocal interactions and heuristic effects.
% LLMs are trained to propose optimization choices while explicitly modeling downstream instruction-count changes, reflecting that good decisions depend on post-compilation effects rather than surface structure . 
% In hardware design, Fang et al. demonstrate an LLM-based Register-Transfer Level (RTL) augmentation pipeline that relies on synthesis-in-the-loop correction to ensure downstream synthesizability and Power, Performance, and Area (PPA) relevance \cite{fang2024transferable}. 
% Relatedly, DeLorenzo et al. introduce Abstractions-of-Thought to mitigate errors induced by abstraction boundaries \cite{delorenzo2025abstraction}, while Veriopt shows that optimizing only for functional correctness neglects critical downstream PPA objectives \cite{tasnia2025veriopt}. 
Results in similarly structured domains such compiler optimization \cite{cummins2023llmcompileropt}, hardware design \cite{fang2024transferable} also support the claim that LLMs struggle with leaky-abstraction design problems. In the case of quantum compilation \cite{Cuccaro2004ripple, javadi2024quantum}, even among valid designs, locally ``better-looking'' structures may invert after transpilation, routing, or scheduling, necessitating multi-level evaluation rather than syntax-level heuristics. The model can be simultaneously (i) confidently fluent while violating hard physical/structural constraints, and (ii) confidently suboptimal because the meaning of a design decision only materializes after lower-level compilation, so without explicit verification and multi-level evaluation in the loop, MLE-trained fluency systematically diverges from both constraint satisfiability and compiled quality.

\begin{SidebarProp}{Syntax–semantics gap}{bg1}
In quantum programming, these ``syntax-first'' conveniences break down, and the syntax–semantics gap widens for two independent reasons:
\begin{itemize}[leftmargin=*, noitemsep, topsep=2pt]
    \item Semantic inspection is intrinsically constrained: common debugging practices such as inspecting intermediate states or step-by-step execution are obstructed because quantum states in general are hard to capture exactly with only classical resources, and on a quantum device, measurement collapses quantum states. In addition, the no-cloning theorem prevents faithful copying, making even the validation of a syntactically correct circuit non-trivial \cite{ramalho2024testing}. 
    \item Even among semantically correct programs, cost semantics leak across abstraction layers, rendering local design judgments unreliable. For example, the ripple-carry adder \cite{Cuccaro2004ripple} admits a clean modular description in terms of MAJ and UMA blocks (see Figure \ref{fig:cuccaro_adder}a), yet apparent module-level improvements (e.g., fewer gates) may invert after gate-level lowering and compiler rewrites such as cancellation, resynthesis, scheduling, and routing. 
    The true optimization objectives are often determined by cross-module interactions that are not apparent at the modular level \cite{javadi2024quantum}. 
    This is a clear case of abstraction leakage, which we discuss in detail in Section \ref{subsec:levels}.
\end{itemize}
\end{SidebarProp}

\begin{figure}[htbp]
    \centering
    \includegraphics[width=\linewidth]{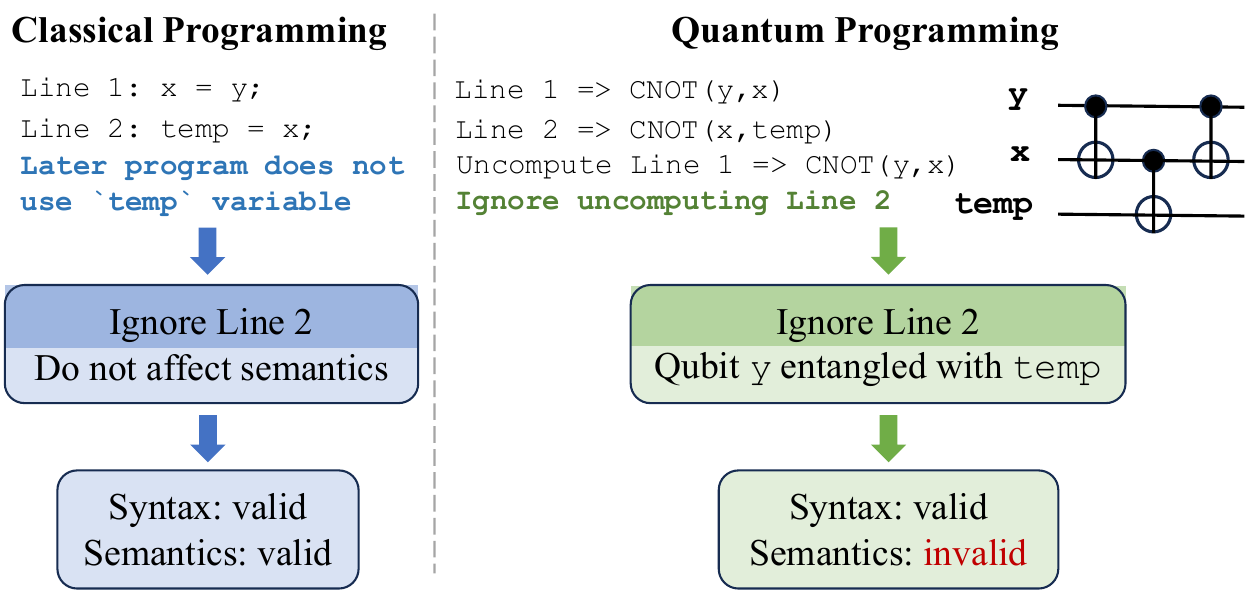}
    \caption{\textbf{Toy example contrasting classical and quantum programming:} syntactic fluency is much less of a predictor of semantic correctness in quantum programming than in classical programming.}
    \label{fig:fig2_comparison}
\end{figure}

\subsection{Structural Bias from Unverified Corpora}\label{subsec:stru}

Training on even lightly curated quantum code biases the learned distribution against physical validity, in a manner structurally analogous to {data poisoning} \cite{tsai2025beyond}. Indeed, practitioners building production quantum-code LLMs already recognize that uncurated public code is unsuitable as training data. In the Qiskit Code Assistant training pipeline \cite{dupuis2024qiskit}, code from official Qiskit GitHub organizations is oversampled by $10.3\times$ relative to general scraped repositories, a ratio chosen empirically for benchmark performance. We take such filtering as a starting point and consider whether the curation strategies currently in use can close the validity gap identified in Section \ref{subsec:gap}. 

Source-based heuristics for curating code data select code that are apparently well-tested, but not code that is guaranteed to implement the intended function. They work in classical settings for an indirect reason: classical code is constantly tested in practice, through unit tests, CI, crashes, and bug reports, so incorrect code gets caught and fixed over time, and high-provenance repositories end up being more trustworthy as a result. Quantum code lacks this tight feedback loop. Incorrect programs are not pruned from public corpora, so filtering by source does little to guarantee correctness, and training on the resulting data carries that gap into the model. Our earlier work \cite{song2026data} reports empirical evidence consistent with this hypothesis: LLMs trained without data verification ceil out at $79\%$ accuracy on quantum circuit optimization tasks, with no further gains from scale.

\subsection{The Phenomenon of Inverse Scaling}

We hypothesize that quantum program generation exhibits a form of inverse scaling \cite{mckenzie2023inverse}: a regime in which improvements in next-token loss do not translate into, and can actively diverge from, improvements in physical validity. Standard scaling laws predict that larger models achieve lower next-token loss on the training distribution \cite{kaplan2020scaling}. In our domain, however, the training distribution itself is biased away from the manifold of valid programs (Section \ref{subsec:stru}), so models that fit it more faithfully at scale fit a biased target more faithfully.

This adaptation manifests at both training and inference time. At training, under maximum likelihood on a biased corpus, the model's distribution $p_\theta$ converges toward the training distribution $p_\text{train}$ rather than toward the distribution $p_\text{valid}$ of physically valid circuits. Since $p_\text{train}$ is structurally misaligned with $p_\text{valid}$, the KL divergence $D_{KL}(p_\text{valid} || p_\theta)$ does not vanish with scale: probability mass concentrates increasingly tightly on syntactically plausible but physically invalid regions of the circuit space. At inference, larger models acquire stronger syntactic priors that are correspondingly harder to override. The Strong Prior modality \cite{mckenzie2023inverse}, which is also demonstrated in classical code \cite{miceli2023larger}, in which large models fail to follow in-context instructions that conflict with patterns learned in pretraining.

% \vspace{-0.2em}
\begin{tcolorbox}[
    colback=blue!4!violet!3!white,
    colframe=blue!30!violet!50,
    boxrule=0.8pt,
    arc=2mm,
    left=8pt, right=8pt, top=8pt, bottom=8pt
]
\textbf{\textsc{Position 1:}} We contend that the objective of generative models in quantum computing must shift from \textit{probabilistic emulation} of program code to \textit{constructive verification} against the mathematical structure of quantum algorithms and subroutines. Scale alone, without verification-aware constraints, cannot bridge the gap between syntax and semantics.
\end{tcolorbox}
% \vspace{-0.8em}

%%%%%%%%%%%%%%%%%%%%%%%%%%%%%%%%%%%%%%%%%%%%%%%%%%%%%%%%%%%%%%%%%%%%%%%%%%%%
%% Section 3: The Complexity
%%%%%%%%%%%%%%%%%%%%%%%%%%%%%%%%%%%%%%%%%%%%%%%%%%%%%%%%%%%%%%%%%%%%%%%%%%%%

\begin{SidebarProp}{Validity}{bg2}
In our context we are concerned with validity of a quantum program on two distinct logical layers:
\begin{itemize}[leftmargin=*, noitemsep, topsep=2pt]
    \item {Structural validity} ($\mathcal{M}_\text{struct}$): Respecting the constraints imposed by the underlying hardware abstraction. For Noisy Intermediate Scale Quantum (NISQ) devices, this reduces to a coupling graph check on physical gates, which is polynomial time in circuit size. For fault-tolerant quantum computing (FTQC) architectures, this expands to a hierarchy of constraints including patch-layout adjacency and boundary-type matching for lattice surgery \cite{litinski2019game,fowler2018low}, magic-state factory throughput and scheduling, code-distance budgeting, and Pauli-frame consistency, several of which are individually NP-hard as combinatorial problems.
    \item {Functional correctness} ($\mathcal{M}_\text{func}$): Implementing the target logical channel within tolerance $\epsilon$, namely the circuit's action on the algorithm's computational register matches some $U_\text{target}\in SU(2^n)$, where $n$ is the logical qubit count. This definition applies to both NISQ and FTQC regimes. Exact verification requires $\mathcal{O}(2^n)$ time in the worst case. Even approximate decision versions such as the non-identity testing (deciding whether a given unitary circuit is close, up to global phase, to the identity of quantum circuits), are known to be computationally hard \cite{watrous2009quantum}. Physical-level concerns specific to fault tolerance, such as decoder convergence and magic-state fidelity, are not in scope for this correctness check.
\end{itemize}
\end{SidebarProp}

\section{The Challenge of Post-selection}
\label{sec:complexity}

Proponents of scaling argue that hallucination can be mitigated through 
sample-and-filter pipelines. This ``Generate-then-Verify'' paradigm 
assumes one can generate $K$ candidates, discard invalid ones and 
recover validity in expectation as $K$ grows \cite{li2022competition}. 
While effective for domains with locally checkable constraints such as 
parser-enforceable syntax \cite{shin2021constrained} or unit-test 
filtering \cite{chen2021evaluating,li2022competition}, this strategy 
fails for generic quantum synthesis. The verification step can be itself 
computationally expensive, and only an exponentially small fraction of 
candidates passes. The sample-and-filter pipeline therefore inherits a 
coupled exponential cost that no amount of model scaling can amortize 
away.

Validity for quantum synthesis decomposes into two layers 
(Background 2): structural validity 
($\mathcal{M}_\text{struct}$), which captures whether a candidate 
respects the architectural rules of the target hardware abstraction, and 
functional correctness ($\mathcal{M}_\text{func}$), which captures 
whether the circuit's logical action matches the target unitary. The 
overall success rate of a sample-and-filter pipeline is the joint 
probability $P(\mathcal{M}_\text{struct})\,P(\mathcal{M}_\text{func} 
\mid \mathcal{M}_\text{struct})$, and the per-candidate verification 
cost is dominated by the cost of deciding membership in 
$\mathcal{M}_\text{func}$.

\subsection{Structural Validity: Cost of Checking and Sampling}\label{sec:div_high}

We analyze $\mathcal{M}_\text{struct}$ along two axes: the per-candidate cost 
of \emph{verifying} membership, and the probability that a candidate drawn 
from a model's distribution $\pi_\theta$ \emph{satisfies} it. The two axes 
scale very differently across regimes. Verification is uniformly cheap: 
each constraint defining $\mathcal{M}_\text{struct}$ is local in the 
candidate description $x$, so for a circuit of $n$ qubits at depth $d$ with 
$G = \Theta(nd)$ operations, membership can be decided in $\Theta(nd)$ time 
by per-operation table lookups, gate-set and coupling-graph adjacency in 
NISQ, and patch-layout, boundary-type, distance-budget, factory-throughput, 
and Pauli-frame consistency in FTQC \citep{litinski2019game,
beverland2022assessing}. The lower bound is immediate by the need to read 
each constrained operation at least once. We treat layout and schedule 
annotations as part of $x$: this scopes $\mathcal{M}_\text{struct}$ to a 
consistency check, separate from the upstream NP-hard problem of 
\emph{producing} such annotations from a logical circuit 
\citep{siraichi2018qubit, tan2024sat}.

Sampling, in contrast, exhibits a sharp regime gap. In NISQ, 
$\mathcal{M}_\text{struct}$ is the conjunction of independent local 
constraints (one per gate), so the pass rate factors as 
$P(x \in \mathcal{M}_\text{struct}^\text{NISQ}) = (1-\epsilon)^G$ in 
the per-gate violation probability $\epsilon$. Supervised fine-tuning 
on hardware-aware data can drive $\epsilon$ below $1/G$, yielding 
$\Theta(1)$ pass rates; this regime is empirically reached by 
production NISQ code-generation systems~\citep{dupuis2024qiskit, 
vishwakarma2024qiskit}. In FTQC, $\mathcal{M}_\text{struct}$ encodes 
\emph{global} combinatorial constraints such as lattice-surgery routes must not 
cross in spacetime, factory throughput must meet T-gate demand at every 
timestep, and distance budgets must sum to the target logical error rate 
\citep{litinski2019game, beverland2022assessing}. These constraints are not separable, and the corresponding feasibility 
problems inherit NP-hardness from the qubit-mapping subroutine 
\citep{siraichi2018qubit, botea2018complexity}, with further evidence of 
practical intractability from SAT-based lattice-surgery synthesis 
\citep{tan2024sat} and Steiner-tree formulations of multi-qubit surgery 
routing \citep{silva2024multi}. As we show in 
Appendix~\ref{app:struct-sampling-hardness}, this implies that no 
polynomial-time sampler $\pi_\theta$, including any LLM with 
$\mathrm{poly}(n)$ parameters and inference cost, can place more than 
$e^{-\Omega(n)}$ mass on $\mathcal{M}_\text{struct}^\text{FTQC}$ in the 
worst case, unless $\mathsf{NP} \subseteq \mathsf{BPP}$.

The implication for the post-selection pipeline is asymmetric. In NISQ, 
structural filtering is essentially free: a $\Theta(1)$ pass rate combined 
with $\Theta(nd)$ verification yields $\Theta(nd)$ expected cost per 
structurally valid candidate. In FTQC, the same pipeline incurs an 
\emph{additional} exponential factor before functional validity is even 
considered, giving expected cost $e^{\Omega(n)} \cdot \Theta(nd)$. This 
worst-case bound is, of course, achievable only on adversarial 
specifications; practical algorithms are built from regular primitives 
(QFT, modular arithmetic, amplitude estimation, Trotterization) whose 
structure admits polynomial-time schedule construction. The point of the 
hardness result is not that valid FTQC compilation is hopeless, but that 
generic probabilistic generation cannot reach it: reliability must come 
from algorithmic structure exploited constructively, not from scale. 
This motivates the verifier-centric paradigm of 
Section~\ref{sec:verifier}.

% \subsection{The Generalization Paradox of Supervised Fine Tuning}
% \label{sec:generalization-paradox}

\begin{figure}
    \centering
    \includegraphics[width=0.8\linewidth]{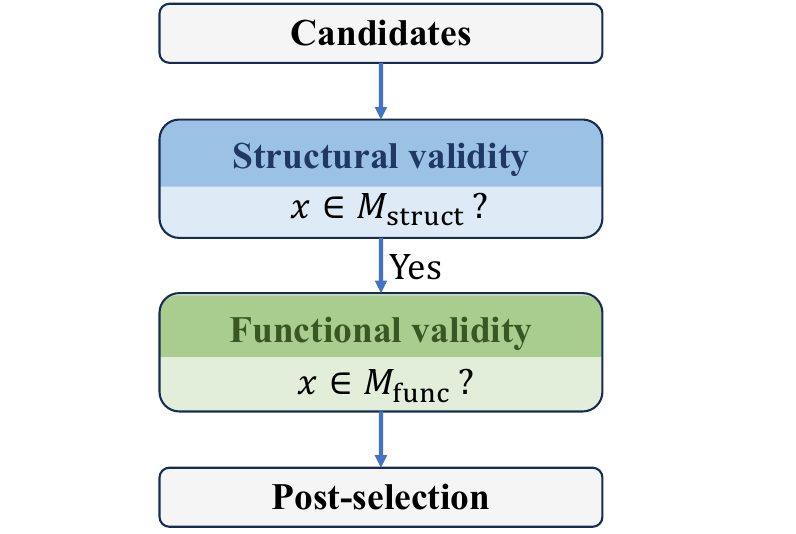}
    \caption{{Two-stage post hoc filtering for quantum synthesis.}}
    \vspace{-2em}
    \label{fig:two-stage_filtering}
\end{figure}

% A common defense involves fine-tuning models on millions of valid circuits that implement common quantum algorithmic primitives such as the quantum Fourier transform and Grover search. This reliance on known algorithms constitutes a \textbf{memorization fallacy}.

% Fine-tuning on structured algorithms allows the model to memorize specific distributions. A robust compiler must generalize to \textbf{out-of-distribution} targets, for instance arbitrary molecule simulations where no textbook decomposition exists. Since the model is trained via maximum likelihood estimation, the objective forces the distribution $p_\theta$ to cover the high-entropy support of the training corpus. When presented with a novel target $U_\text{new}$, the model lacks a gradient signal to concentrate probability on the unique correct circuit. Instead, it reverts to the maximum-entropy prior learned from diverse but irrelevant training examples.

\subsection{Functional Validity: Cost of Checking and Sampling}
\label{sec:func-validity}
The $\Theta(2^n)$ verification cost in Background 2, namely the cost of comparing a circuit's induced unitary $U(x)$ against 
$U_\text{target}$, is not improved by tensor-network methods: matrix 
product states \citep{vidal2003efficient} require bond dimension that grows 
exponentially for the volume-law-entangled circuits characteristic of 
post-classical algorithms \citep{schuch2008entropy}. We therefore work with 
$C_\text{func} = \Theta(2^n)$ for generic synthesis; 
Appendix~\ref{app:func-verification} discusses structured-target escape 
routes (Clifford fragments, ZX-calculus normalization).

Verification cost is only half the problem. The conditional pass rate 
$P(x \in \mathcal{M}_\text{func} \mid x \in \mathcal{M}_\text{struct})$ 
also decays exponentially in $n$, by an argument whose form depends on 
the program representation: Lipschitz volume counting on an $O(nd)$-dimensional 
parameter manifold inside $SU(2^n)$ for NISQ ansatzes 
\citep{nielsen2010quantum}, or direct enumeration over a combinatorial 
program space of size $e^{\Theta(nd \log n)}$ for FTQC instruction 
sequences. Both yield, under the non-pathology assumption that 
$\pi_\theta$ does not memorize the specific solution,
\begin{equation}
\label{eq:func-pass}
P_{x \sim \pi_\theta}\!\left(\|U(x) - U_\text{target}\|\!\leq\!\varepsilon 
\middle| x \in \mathcal{M}_\text{struct}\right) 
\!\leq\! C_\theta \cdot e^{-\gamma n},
\end{equation}
where $C_\theta$ measures how far $\pi_\theta$ deviates from the uniform 
reference distribution on its support, bounded by $\mathrm{poly}(n)$ 
unless the model has memorized the specific solution, and $\gamma$ is 
polynomial in $d$ and logarithmic in $n$ or $1/\varepsilon$ 
(proofs in Appendix~\ref{app:func-sparsity}).\footnote{Two scope 
conditions are worth stating explicitly. (a) The NISQ bound requires 
standard regularity of the parameterization (full-rank Jacobian, 
bounded fiber multiplicity); pathological ansatzes in the 
barren-plateau regime \citep{mcclean2018barren} violate this, but in 
that regime synthesis already fails for optimization-theoretic 
reasons independent of our bound. (b) The bound is pointwise in 
$U_\text{target}$ but informative only when $C_\theta = 
\mathcal{O}(\mathrm{poly}(n))$, which holds by definition of 
generalization for out-of-distribution targets. For in-distribution 
algorithmic primitives (textbook QFT, Grover, modular arithmetic), 
$C_\theta$ can be large enough to cancel $e^{-\gamma n}$ entirely, 
which is consistent with empirical success of code-completion 
systems on such benchmarks \citep{dupuis2024qiskit, 
vishwakarma2024qiskit}. The validity gap is therefore not about 
whether circuits can in principle implement $U_\text{target}$, but 
about whether a sampler can find one without already having seen the 
answer.} The implication for the post-selection pipeline of 
Section~\ref{sec:exponential-cost} is direct: functional verification 
costs $\Theta(2^n)$ per candidate while only an $e^{-\gamma n}$ fraction 
of candidates passes it, so the two exponentials compound.

\subsection{The Exponential Cost of Post-selection}\label{sec:exponential-cost}

Combining the structural and functional bounds of 
Sections~\ref{sec:div_high} and~\ref{sec:func-validity} gives the total 
expected cost of the post-hoc selection pipeline 
(Figure~\ref{fig:two-stage_filtering}). Writing
\begin{equation*}
\begin{aligned}
p_\text{struct} &:= P(x \in \mathcal{M}_\text{struct}), \\
p_{\text{func}\mid\text{struct}} &:= P(x \in \mathcal{M}_\text{func} \mid x \in \mathcal{M}_\text{struct}),
\end{aligned}
\end{equation*}
for the structural and conditional functional pass rates analyzed in 
Sections~\ref{sec:div_high} and~\ref{sec:func-validity} respectively, 
and with per-candidate verification costs $C_\text{struct}\in\Theta(nd)$ 
and $C_\text{func}\in\Theta(2^n)$, the expected cost to obtain one 
functionally correct candidate is:
\begin{equation}
\label{eq:total-cost}
    \mathbb{E}[C] 
    \;=\; 
    \frac{C_\text{struct}}{p_\text{struct} \cdot p_{\text{func}\mid\text{struct}}}
    \;+\;
    \frac{C_\text{func}}{p_{\text{func}\mid\text{struct}}}.
\end{equation}
We substitute the regime-specific bounds established earlier: 
$p_\text{struct}^{\text{NISQ}} = \Theta(1)$ versus 
$p_\text{struct}^{\text{FTQC}} \leq e^{-\beta n}$ from 
Proposition~\ref{prop:ftqc-sampling-app}, where $\beta > 0$ denotes the 
implicit constant in its $\Omega(n)$ exponent; and 
$p_{\text{func}\mid\text{struct}} \leq C_\theta \cdot e^{-\gamma n}$ from 
Equation~\eqref{eq:func-pass}, with regime-specific exponents 
$\gamma_\text{NISQ} = \Omega(d\log(1/\varepsilon))$ and 
$\gamma_\text{FTQC} = \Omega(d\log n)$ from 
Propositions~\ref{prop:func-sparsity-nisq} 
and~\ref{prop:func-sparsity-ftqc}. This yields
\begin{align}
\label{eq:nisq-cost}
\mathbb{E}[C]^{\text{NISQ}}  &= \mathcal{O}\!\left(e^{(\ln 2 + \gamma_\text{NISQ})\, n}\right), \\
\label{eq:ftqc-cost}
\mathbb{E}[C]^{\text{FTQC}}  &= \mathcal{O}\!\left(e^{(\max(\ln 2,\, \beta) + \gamma_\text{FTQC})\, n}\right).
\end{align}
NISQ couples two exponentials (simulation cost $\ln 2$ plus functional 
sparsity $\gamma_\text{NISQ}$); FTQC couples three, with structural 
sparsity $\beta$ contributing alongside $\gamma_\text{FTQC}$ and 
combining with $\ln 2$ via the max --- whichever of structural sampling 
or functional verification is more expensive dominates. Attempting to 
bypass these bounds via matrix product states fails for 
quantum-advantage regimes: useful circuits exhibit volume-law 
entanglement, forcing the bond dimension to grow exponentially 
\citep{vidal2003efficient, schuch2008entropy}. Specialized verifiers 
exist for structured subclasses (Clifford, local Hamiltonian, Clifford+T 
via ZX; see Appendix~\ref{app:func-verification}) but not for the 
generic synthesis problem.

\begin{tcolorbox}[
    colback=blue!4!violet!3!white,
    colframe=blue!30!violet!50,
    boxrule=0.8pt,
    arc=2mm,
    left=8pt, right=8pt, top=8pt, bottom=8pt
]
\textbf{\textsc{Position 2:}} Probabilistic scaling cannot close the
coupled exponential gap of Equation~\eqref{eq:total-cost}: structural
sparsity in FTQC and functional sparsity in both regimes compound with
simulation cost. We assert that validity must be enforced
\emph{constructively} during candidate generation, not \emph{filtered}
after it, shifting the verifier from an output gate to a generation
invariant.
\end{tcolorbox}

%%%%%%%%%%%%%%%%%%%%%%%%%%%%%%%%%%%%%%%%%%%%%%%%%%%%%%%%%%%%%%%%%%%%%%%%%%%%
%% Section 4: The Verifier-Centric Paradigm
%%%%%%%%%%%%%%%%%%%%%%%%%%%%%%%%%%%%%%%%%%%%%%%%%%%%%%%%%%%%%%%%%%%%%%%%%%%%

\section{The Verifier-centric Paradigm}
\label{sec:verifier}

Current industry copilots optimize for human-centered objectives \cite{barke2023grounded,vaithilingam2022expectation,liang2024understanding}: variable naming, commenting, and syntactic sugar. While valuable for classical software where the human serves as the ultimate reviewer, this alignment constitutes a misallocation of resources for quantum program generation. %At the scale of quantum advantage ($n > 50$), the Hilbert space dimension exceeds $10^{15}$ \cite{preskill2018quantum,arute2019quantum}, rendering intuitive human verification impossible. 
The breakdown of human verification is not asymptotic but concrete: classical state-vector simulation, the most direct verification aid available to a human reviewer, becomes infeasible beyond roughly 40–50 qubits on current hardware \cite{dalzell2020howmany,arute2019quantum}, and recent experiments on IBM's 127-qubit Eagle experiments \cite{kim2023evidence} already operate well outside this regime. At these scales, the combination of syntax-semantics gap (Section \ref{sec:data_crisis}), the exponential sparsity of valid circuits (Section \ref{sec:complexity}), and the fundamental quantum properties such as no-cloning makes manual deduction of circuit correctness intractable for any human reviewer. Human supervision becomes a bottleneck rather than a safeguard.
% Due to the uniquely quantum features mentioned so far including the syntax-semantics gap (Section \ref{sec:data_crisis}) and exponential sparsity of valid circuits (Section \ref{sec:complexity}), combined with fundamental properties of quantum mechanics such as no-cloning, the complexity of genuinely post-classical quantum circuits will be beyond the grasp of human engineers to manually deduce their correctness. Human supervision becomes a bottleneck rather than a safeguard.

We posit that the objective of AI tools for rigorous science must invert. The target audience is not the human developer but the \emph{formal verifier}. The agent functions as a high-throughput heuristic proposer for the solver, operating at machine speed and unconstrained by human verification bandwidth. 

% \begin{tcolorbox}[
%     colback=gray!5!white,
%     colframe=black,
%     boxrule=0.8pt,
%     arc=0mm,
%     title=\textbf{\textsc{Paradigm Shift: The Verifier-Centric Architecture}},
%     fonttitle=\bfseries\small,
%     left=6pt, right=6pt, top=6pt, bottom=6pt
% ]
% \small
% \textbf{The Probabilistic Copilot (Current State)}
% \begin{itemize}[leftmargin=1.2em, noitemsep, topsep=2pt]
%     \item \textit{Objective:} Maximize likelihood $P(\text{token} | \text{human\_corpus})$.
%     \item \textit{Verifier:} Post-hoc, manual, or unit-test based.
%     \item \textit{Failure Mode:} Confident hallucination of plausible syntax.
% \end{itemize}
% \vspace{4pt}
% \hrule
% \vspace{4pt}
% \textbf{The Verifier-Centric Agent (Our Proposal)}
% \begin{itemize}[leftmargin=1.2em, noitemsep, topsep=2pt]
%     \item \textit{Objective:} Maximize $\mathbb{E}[\text{Quality}]$ subject to $\Verifier(c) \models \text{True}$.
%     \item \textit{Verifier:} Intrinsic, hierarchical, and solver-in-the-loop.
%     \item \textit{Mechanism:} Constructive generation where constraints are preconditions for token emission.
% \end{itemize}
% \end{tcolorbox}

% \begin{figure}[htbp]
%     \centering
%     \includegraphics[width=0.95\linewidth]{figs/Paradigm_Shift_from_Probabilistic_Copilots_to_Verifier-Centric_Agents.pdf}
%     \caption{Paradigm shift from probabilistic copilots to verifier-centric agents.}
%     \label{fig:paradigm}
% \end{figure}

\subsection{Constructive Verification Protocols}
\label{subsec:construct}

To surmount the exponential cost of post-selection identified in Section~\ref{sec:complexity}, we propose replacing blind generation with {constructive verification} (Figure \ref{fig:markov}). This framework adapts the {constrained decoding} principles \cite{shin2021constrained} to domains of rigorous sciences, which have been proven effective in semantic parsing tasks like SQL generation \cite{scholak2021picard}.

We formalize this process as a trajectory search through the discrete syntactic space $\mathcal{X}$. 
The iterative abstract quantum circuit/program design process is modeled as a Markov Decision Process $\mathcal{M}=(\mathcal{S},\mathcal{A},P)$ where a policy is a conditional distribution
$\pi_\theta(a_t|s_t)$ that samples an action $a_t\in\mathcal{A}$ at time step $t$ given the current state $s_t\in\mathcal{S}$ and the environment dynamics are given by a transition kernel
$P(s_{t+1}| s_t,a_t)$ \cite{puterman1994mdp,sutton2018rl}.

\begin{itemize}[leftmargin=*, noitemsep, topsep=-2pt]
  \item \textbf{State space $\mathcal{S}$.}
  A state $s_t$ is the agent's current {design snapshot}, denoted as
  \[
    s_t := \big(x_t,\;\phi_t,\;\eta_t,\;\kappa_t\big),
  \]
  where $x_t\in\mathcal{X}$ is the current abstract program/circuit is in a fixed intermediate representation (IR),
  $\phi_t$ has current formal parameters (e.g., input size $n$, error budgets, success probability targets),
  $\eta_t$ stores the latest evaluation outputs/metrics (depth, $T$-count, ancilla count, asymptotic scaling estimates),
  and $\kappa_t$ has any auxiliary bookkeeping needed to make the process Markov (e.g., tool call traces, module-commitment flags).

  \item \textbf{Action space $\mathcal{A}$.}
  Each action corresponds to invoking a \emph{tool} from a fixed library with parameters, e.g.
  \[
    a_t \in \{\textsc{Edit}[\rho,\alpha],\;\textsc{Eval},\;\textsc{Finish}\},
  \]
  where $\textsc{Edit}[\rho,\alpha]$ applies a property-preserving rewrite schema $\rho$ with arguments $\alpha$
  (chosen from a verified transformation library), see Figure \ref{fig:cuccaro_adder}c.
  $\textsc{Eval}$ re-checks correctness and evaluates metrics,
  and $\textsc{Finish}$ terminates and returns the final template and summary.

  \item \textbf{Dynamics $P(s_{t+1}|s_t,a_t)$.}
  The transition kernel is induced by the tool execution and its {guardrails}:
  \[
    s_{t+1} \sim P(\cdot|s_t,a_t) \;\;\Leftrightarrow\;\; s_{t+1}=\textsc{ToolStep}(s_t,a_t).
  \]
  Concretely, for an edit action $a_t=\textsc{Edit}[\rho,\alpha]$ for some $\rho$ and $\alpha$, $\textsc{ToolStep}(s_t,a_t)=\big(x_t',\phi_t,\eta_t,\kappa_t'\big)$ if the rewrite is admitted and applied, and $\textsc{ToolStep}(s_t,a_t)=\big(x_t,\phi_t,\eta_t,\kappa_t''\big)$ if the action is blocked/rejected (no-op).

  \item \textbf{Validity as a {state-dependent admissible action set}.}
  The semantic constraints discussed in Sections \ref{sec:data_crisis} and \ref{sec:complexity} imply that, for each state $s_t$, there is
  an admissible action subset $\mathcal{A}_{\mathrm{valid}}(s_t)\subseteq\mathcal{A}$ to ensure that $x_{t+1}\in\mathcal{M}_\text{func}$.
  We implement this through a \emph{validity filter} (a.k.a.\ shield / mask)
  \[
    \mathcal{V}(a,s) \in \{0,1\},\quad
    \mathcal{A}_{\mathrm{valid}}(s):=\{a\in\mathcal{A}:\mathcal{V}(a,s)=1\}.
  \]
  The filter may be composed of multiple components and dependent on the action $a_t$ taken. For instance, if $a_t=\textsc{Edit}(\text{Parallelize},(g_1,g_2))$ for some chosen pair of gates $g_1$, $g_2$ in the circuit $x_t$, namely if the action is to try parallelizing $g_1$ and $g_2$, the validity filter needs to check 1) that the two gates operate on distinct sets of qubits ("no-overlap") and 2) there is no obstruction for the two gates to be moved to the same layer ("no-obstacle"). In this case the validity filter takes the form of
  \begin{equation}\label{eq:filter_comp}
    \mathcal{V}(a,s) \;=\; \mathcal{V}_{\text{no-overlap}}(a,s)\cdot \mathcal{V}_{\text{no-obstacle}}(a,s),
  \end{equation}
  Here we expect each component of the validity filter to run in poly$(n)$ time. The components of the validity filters in \eqref{eq:filter_comp} clearly run in polynomial time.

  \item \textbf{Constrained policy.}
  Let $\pi_\theta^{(0)}(a|s)$ be the {unconstrained} LLM-induced proposal over tool calls.
  The constructive-verification policy is proportional to the unconstrained proposal multiplied by the validity filter (up to normalization):
  \begin{equation}\label{eq:cons_pol}
    \pi_\theta(a|s)
    \;\propto\;
    \pi_\theta^{(0)}(a|s)\;\cdot\;\mathcal{V}(a,s).
  \end{equation}
  The product form of the constrained policy in \eqref{eq:cons_pol} is reminiscent of the standard invalid-action masking construction in reinforcement learning \cite{huang2020masking,hou2023masking}.
\end{itemize}

\begin{figure}
    \centering
    \includegraphics[width=\linewidth]{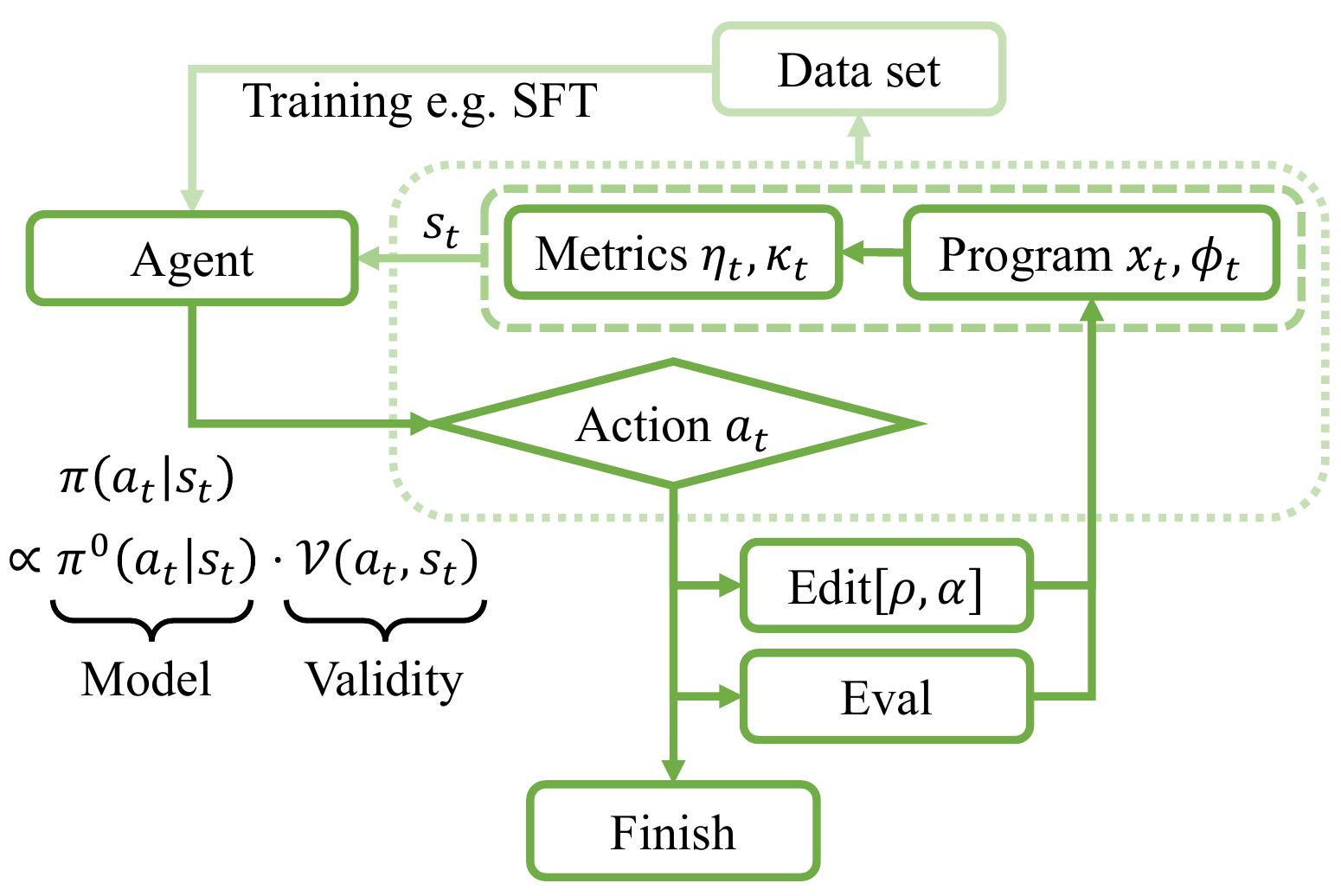}
    \caption{\textbf{Control flow of the verification protocol.} This can be considered as an expanded view of Figure \ref{fig:fig2_comparison} where the decision process of the verifier-centric agent is expanded in greater detail.}
    \label{fig:markov}
\end{figure}

\subsection{Multiple Levels of Abstraction}
\label{subsec:levels}

\begin{figure*}[htbp]
    \centering
    \includegraphics[width=\linewidth]{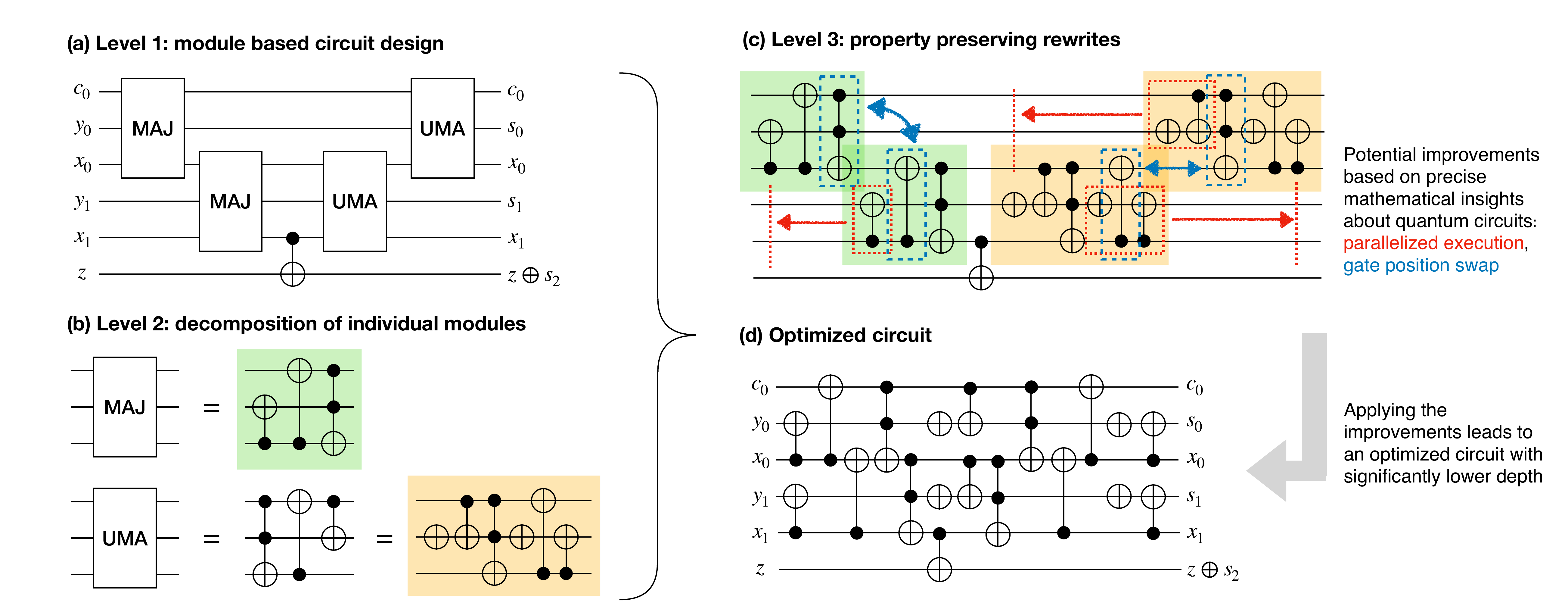}
    \caption{Illustration of the different levels of abstraction using the Cuccaro Adder \cite{Cuccaro2004ripple} as an example. \textbf{(a)} Level 1 in the hierarchy describes the abstract circuit design based on modules MAJ and UMA. Here a concrete example with $n=2$ is shown. \textbf{(b)} Level 2 expands each module into more elementary gate operations. \textbf{(c)} Level 3 identifies opportunities for improvements and acts on those opportunities with property-preserving rewrites $\rho$. Here $\rho\in\{\text{Parallelize}, \text{Swap}\}$ and the parameter set $\alpha$ passed depends on which $\rho$. \textbf{(d)} Applying the improvements (parallelized execution of gates and gate position swap due to commutativity) yields an optimized circuit with lower depth. We also provide a detailed example of how such property perserving rewrites can be applied to the Cuccaro ripple-carry adder in Appendix \Cref{sec:rewrite}.}
    \label{fig:cuccaro_adder}
\end{figure*}

Even though the semantics of an arbitrary $n$-qubit circuit quickly becomes intractable to visually or exhaustively validate as $n$ grows, practical quantum programs are rarely arbitrary: they are typically assembled from a small vocabulary of algorithmic primitives with a strong subroutine containment structure. As an example, the circuit for Shor's algorithm \citep[Chapter 5]{nielsen2010quantum} is built from quantum fourier transform (QFT) and modular exponentiation blocks. Modular exponentiation then decomposes into modular multiplication and modular addition, each with well-understood reversible implementations and local invariants (e.g., ancilla management and uncomputation patterns) that can be verified and constrained module-by-module rather than by reasoning about an entire unitary monolithically. This multi-level containment is the key escape hatch from the brute-force complexity discussed in Section \ref{sec:complexity}: instead of handling a large quantum circuit as a monolith, we validate and constrain generation at the granularity of primitives (QFT, modular addition/multiplication/exponentiation, etc.), and then compose those certificates—exactly the kind of hierarchical abstraction that a verifier-centric agent must exploit.

\textbf{Level 1: module based circuit design.}
At the top level, the agent manipulates a {parametric} quantum program with $n$ treated as a formal parameter (alongside other algorithmic parameters), and subroutines are invoked as black boxes.  Verification at this level is therefore {structural}: the agent checks that the correct high-level subroutine calls appear in the correct order and with consistent wiring of symbolic registers, without expanding the subroutines into gates.
For the Cuccaro ripple-carry adder \cite{Cuccaro2004ripple}, the $n$-bit in-place addition can be written as a short program over two 3-bit modules, $\mathrm{MAJ}$ and $\mathrm{UMA}$, applied in forward/backward sweeps (Figure \ref{fig:cuccaro_adder}a). At this abstraction, the agent reasons about {containment} where the adder is composed of $\mathrm{MAJ}/\mathrm{UMA}$ modules.

% \textbf{Level 2: topological compliance.} Before emitting a gate $U_{ij}$, the agent queries the hardware graph $G=(V, E)$. If $(i, j) \notin E$, the token is masked ($\mathbb{I}=0$). This enforces the topological manifold $\mathcal{M}_\text{topo}$ strictly at generation time. Unlike post-hoc transpilation which inserts SWAPs heuristically, the generative model learns to {route} logic through the topology natively.

\textbf{Level 2: decomposition of individual modules.}
Level~2 is the first level where the focus shifts to circuits on $O(1)$ qubits: each black-box module is compiled into a concrete gate-level implementation.
For the adder, this means choosing decompositions of the 3-qubit blocks
$\mathrm{MAJ}$ and $\mathrm{UMA}$ into elementary gates
(e.g., CNOT/Toffoli or backend-native one- or two-qubit gates), exactly as illustrated in Figure \ref{fig:cuccaro_adder}b.
This is the regime where typical quantum gate transpilers operate: correctness obligations are localized to a small footprint, and the agent can use standard synthesis/transpilation toolchains \cite{javadi2024quantum} to map each module to a target gate set and connectivity.

% \textbf{Level 3: symbolic heuristics.} Within a module, the agent utilizes lightweight symbolic checkers such as ZX-calculus \cite{coecke2011interacting,kissinger2020pyzx} or Clifford frame tracking \cite{gottesman1998heisenberg}. While incomplete for universal non-Clifford circuits, these methods serve as a {coarse-grained filter}. They efficiently detect the majority of local invariant violations (e.g., phase mismatch, Pauli frame errors) in polynomial time. Full state-vector simulation is reserved only for the final verification of small, committed blocks that pass these symbolic checks.

\textbf{Level 3: elementary-operation optimization (property-preserving rewrites).}
At Level~3, the program is treated as a monolithic sequence of elementary operations obtained by {inlining} the Level 2 decompositions into the Level 1 program.  The agent then searches for depth/gate-count improvements via {property-preserving rewrites} and local commutation or parallelization opportunities (Figure \ref{fig:cuccaro_adder}c): e.g., swapping commuting gates, canceling inverse pairs, sliding CNOTs across Toffolis when allowed, and parallelizing disjoint operations.  Crucially, Level~3 remains aware of higher-level choices: selecting a different Level~2 decomposition (or changing the Level~1 module schedule) can expose or eliminate rewrite opportunities, so optimization is inherently cross-level even though the rewrite rules act on elementary gates. 

To illustrate how the control framework outlined in Section \ref{subsec:construct} can be applied across multiple layers of abstraction (Figure \ref{fig:cuccaro_adder}), a detailed worked example is shown in \Cref{sec:rewrite} in the Appendix. Here the goal is to demonstrate the form of the program $x_t$ and how property preserving rewrites can help the agent work through different layers of abstraction.

\subsection{The Data Imperative}
\label{subsec:data}

Our earlier work \cite{song2026data} argued for verified training data as an architectural primitive rather than a curation afterthought. The present work supplies its complexity foundation. Standard code-generation pipelines rely on syntactic supervision, which suffices for classical software but collapses in scientific domains where syntactic fluency is decoupled from physical semantics (\Cref{sec:data_crisis}). A model trained on such data mimics how humans write code while ignoring the underlying constraints. We therefore advocate trace-based datasets, analogous to process reward models in mathematical reasoning \cite{lightman2023lets}. Data for scientific agents must contrast optimal solutions against valid but sub-optimal candidates, encoding the optimization manifold rather than a binary correctness function.

Narrowing the syntax–semantics gap forces the agent to internalize hard domain constraints, not just surface syntax. This is the broader AI4Science mandate. Scientific data is scarce, so inductive bias must be supplied a priori. Otherwise, scaling only produces hallucinations at a larger scale.

% \vspace{-0.2em}
\begin{tcolorbox}[
    colback=blue!4!violet!3!white,
    colframe=blue!30!violet!50,
    boxrule=0.8pt,
    arc=2mm,
    left=8pt, right=8pt, top=8pt, bottom=8pt
]
\textbf{\textsc{Position 3:}} AI4Science must transition from open-loop generation to closed-loop constructive verification. We contend that integrating constraints across multiple levels of abstraction is a viable path to scale modular quantum synthesis beyond the limits of human cognition.
\end{tcolorbox}
% \vspace{-0.8em}

%%%%%%%%%%%%%%%%%%%%%%%%%%%%%%%%%%%%%%%%%%%%%%%%%%%%%%%%%%%%%%%%%%%%%%%%%%%%
%% Section 5: Alternative Views
%%%%%%%%%%%%%%%%%%%%%%%%%%%%%%%%%%%%%%%%%%%%%%%%%%%%%%%%%%%%%%%%%%%%%%%%%%%%

\section{Alternative Views}
\label{sec:alternative_views}

This paper proposes to shift from human-centric copilots to verifier-centric agents, which challenges the generative AI (GenAI) orthodoxy. Below, we address three common counter-arguments. We contrast GenAI wisdom with the strict requirements of quantum program generation.

\textbf{View 1: Scaling maximalist.} In the machine learning community, we often hear that the model with enough parameters will have an ``Aha'' moment to understand the logic of events or physics rules (quantum circuits) purely \cite{wei2022emergent}. Many researchers argue that specialized verification is unnecessary. They posit that with sufficient parameter scaling, models also will eventually ``grok'' the logic of quantum circuits purely by training on vast corpora of QASM and Python code.

\textbf{Rebuttal: The inverse scaling of quantum correctness.} This extrapolation conflates natural language fluency with physical validity. Large models often exhibit inverse scaling on tasks with strict logical constraints \cite{mckenzie2023inverse, zhou2024larger}, becoming more confident in their hallucinations \cite{dziri2023faith}. A larger model may emit syntactically perfect QASM yet fail to implement the target unitary due to subtle phase errors or invalid gate decompositions (see \Cref{sec:data_crisis} and \Cref{fig:fig2_comparison}). Scaling amplifies stylistic mimicry while remaining blind to Hilbert-space constraints, and the same pathology surfaces at inference, where scaling test-time compute pushes models into unfocused exploration \cite{gema2025inverse}. The admissible action set $\mathcal{A}_{\text{valid}}(s)$ in \Cref{sec:verifier} rules out both failure modes by construction.

%%%

\textbf{View 2: Developer experience advocate.} A common objection from the software engineering and human--AI interaction communities holds that coding assistants must remain legible and interactive to the human developer, even if final correctness adjudication is delegated to a solver. Optimizing purely for a formal verifier risks producing opaque artifacts that erode developer trust.

\textbf{Rebuttal: Legibility lives in the trace, not the artifact.} We do not eliminate the human; we relocate their role from correctness adjudicator (already intractable beyond $\sim$50 qubits, \Cref{sec:verifier}) to specification author and trace auditor. The trace-based datasets advocated in \Cref{subsec:data} record \emph{why} each rewrite was applied and \emph{which alternatives were rejected}, providing a richer interpretability surface than the final circuit. The human reads the reasoning, not the unitary --- strictly more informative than syntactic copilots, which provide neither correctness guarantees nor process-level traces \cite{lightman2023lets}.

%%%

\textbf{View 3: Computational pragmatist.} Critics might argue that running formal verifiers during generation is computationally too expensive, preferring the less costly inference of standard deep learning models followed by post-hoc filtering.

\textbf{Rebuttal: The economics of quantum discovery.} We argue that the ``Generate then Filter'' paradigm is mathematically ruinous for quantum search spaces, or generally in domains where valid solutions are rare (rejection sampling is notoriously inefficient). AlphaCode required over one million samples to solve hard competitive problems \cite{li2022competition}. Bio-inspired metaheuristics \cite{song2024sais} hit the same exponential wall.
In quantum computing, this inefficiency is exponentially worse. As shown in \Cref{sec:complexity}, the valid subspace decays as $e^{-\gamma n}$. Brute-force sampling is not just expensive, it is intractable compared to verification-aware generation.

%%%

\paragraph{View 4: Quantum verification is a solved problem.}
Mahadev~\citep{mahadev2018classical} gives a classical verifier for arbitrary \textsf{BQP} computations under post-quantum cryptographic assumptions, with follow-up extending the protocol to blind, succinct, and non-interactive settings~\citep{fitzsimons2017unconditionally,reichardt2013classical,bartusek2022succinct}. If verification is already solved, why claim it as the bottleneck?

\paragraph{Rebuttal: Execution verification is not functional verification.}
That line verifies a prover \emph{honestly executed a specified circuit} $C$. Our concern is the orthogonal problem: given a functional specification $U_{\mathrm{target}}$, decide whether a candidate $C$ implements it. Mahadev's protocol takes $C$ as input and is silent on whether $C$ was the right program in the first place. This is the gap between $\mathcal{M}_{\mathrm{struct}}$ and $\mathcal{M}_{\mathrm{func}}$ in Section~\ref{sec:verifier}; the \textsf{QMA}-completeness of non-identity testing~\citep{watrous2009quantum} applies directly to the latter and is not lifted by cryptographic execution verification.

%%%%%%%%%%%%%%%%%%%%%%%%%%%%%%%%%%%%%%%%%%%%%%%%%%%%%%%%%%%%%%%%%%%%%%%%%%%%
%% Section 6: Conclusion
%%%%%%%%%%%%%%%%%%%%%%%%%%%%%%%%%%%%%%%%%%%%%%%%%%%%%%%%%%%%%%%%%%%%%%%%%%%%

\section{Conclusion}
\label{sec:conclusion}

We have argued that probabilistic scaling alone is unlikely to close the validity gap in quantum program generation. Training on unverified code introduces a structural bias against physical validity, and the exponential sparsity of functionally correct circuits makes post-hoc filtering intractable for generic synthesis. Larger models trained this way fit a biased distribution more faithfully, rather than converging on the manifold of valid programs.

We therefore advocate a shift from generate-then-filter pipelines to verifier-centric generation. Solving arbitrary circuits remains out of reach, but the hierarchical, module-based approach of Section~\ref{sec:verifier} is tractable for the broad class of algorithms built from well-understood primitives, by enforcing topological and symbolic constraints during decoding. Trace-based datasets, which record why a rewrite was chosen and why alternatives were rejected, offer a natural training signal for this regime.

% \vspace{-0.2em}
\begin{tcolorbox}[
    colback=blue!4!violet!3!white,
    colframe=blue!30!violet!50,
    boxrule=0.8pt,
    arc=2mm,
    left=8pt, right=8pt, top=8pt, bottom=8pt
]
\centering
\textbf{In domains with hard physical and mathematical constraints, generative models should be coupled to constructive verification rather than relying on scale to recover validity after the fact.}
\end{tcolorbox}

\section*{Impact Statement}
Scaling on corpora will not close the gap between fluency and correctness when validity is governed by mathematical structure rather than statistical regularity. Larger models trained on unverified quantum code yield more confident hallucinations, not useful software. The same failure recurs wherever valid outputs occupy a vanishing fraction of a high-dimensional space, including protein design, theorem proving, and chip layout. Progress there will come from smaller systems coupled to formal checkers, not from larger imitators. A verifier-centric paradigm relocates trust onto the verifiers themselves. A subtle solver bug can silently certify invalid designs at scale, so auditing the verifier stack must sit alongside model training as a research priority. Constraint-aware generation makes outputs easier to audit, which reduces rather than increases risk.

\section*{Acknowledgements}
Many thanks to the Imperial College Centre for Cryptocurrency Research and Engineering for research support. Thanks to Runzhou Tao (University of Maryland) for early discussion. Junhao Song is a first-year PhD student fully funded by the Hitachi-Imperial Centre for Decarbonisation and Natural Climate Solutions, a collaboration between Hitachi Ltd, Hitachi Europe and Imperial College London.

\newpage
\bibliography{example_paper}

@article{austin2021program,
  title={Program Synthesis with Large Language Models},
  author={Austin, Jacob and Odena, Augustus and Nye, Maxwell and Bosma, Maarten and Michalewski, Henryk and Dohan, David and Jiang, Ellen and Cai, Carrie and Terry, Michael and Le, Quoc and others},
  journal={arXiv preprint arXiv:2108.07732},
  year={2021},
  doi={10.48550/arXiv.2108.07732}
}

@article{chen2021evaluating,
  title={Evaluating Large Language Models Trained on Code},
  author={Chen, Mark and Tworek, Jerry and Jun, Heewoo and Yuan, Qiming and Pinto, Henrique Ponde de Oliveira and Kaplan, Jared and Edwards, Harri and Burda, Yuri and Joseph, Nicholas and Brockman, Greg and others},
  journal={arXiv preprint arXiv:2107.03374},
  year={2021},
  doi={10.48550/arXiv.2107.03374}
}

@inproceedings{dupuis2024qiskit,
  title={Qiskit code assistant: Training {LLMs} for generating quantum computing code},
  author={Dupuis, Nicolas and Buratti, Luca and Vishwakarma, Sanjay and Forrat, Aitana Viudes and Kremer, David and Faro, Ismael and Puri, Ruchir and Cruz-Benito, Juan},
  booktitle={2024 IEEE LLM Aided Design Workshop (LAD)},
  pages={1--4},
  year={2024},
  organization={IEEE},
  doi={10.1109/LAD62341.2024.10691762}
}

@article{ruiz2025quantum,
  title={Quantum circuit optimization with {AlphaTensor}},
  author={Ruiz, Francisco JR and Laakkonen, Tuomas and Bausch, Johannes and Balog, Matej and Barekatain, Mohammadamin and Heras, Francisco JH and Novikov, Alexander and Fitzpatrick, Nathan and Romera-Paredes, Bernardino and van de Wetering, John and others},
  journal={Nature Machine Intelligence},
  pages={1--12},
  year={2025},
  publisher={Nature Publishing Group UK London}
}

@article{romeraparedes2024mathematical,
  title = {Mathematical discoveries from program search with large language models},
  author = {Romera-Paredes, Bernardino and others},
  journal = {Nature},
  volume = {625},
  number = {7995},
  pages = {468--475},
  year = {2024},
  doi = {10.1038/s41586-023-06924-6}
}

@article{vishwakarma2024qiskit,
  title={{Qiskit} {HumanEval}: An Evaluation Benchmark For Quantum Code Generative Models},
  author={Vishwakarma, Sanjay and Harkins, Francis and Golecha, Siddharth and Bajpe, Vishal Sharathchandra and Dupuis, Nicolas and Buratti, Luca and Kremer, David and Faro, Ismael and Puri, Ruchir and Cruz-Benito, Juan},
  journal={arXiv preprint arXiv:2406.14712},
  year={2024},
  doi={10.48550/arXiv.2406.14712}
}

@article{kaplan2020scaling,
  title={Scaling laws for neural language models},
  author={Kaplan, Jared and McCandlish, Sam and Henighan, Tom and Brown, Tom B and Chess, Benjamin and Child, Rewon and Gray, Scott and Radford, Alec and Wu, Jeffrey and Amodei, Dario},
  journal={arXiv preprint arXiv:2001.08361},
  year={2020},
  doi={10.48550/arXiv.2001.08361}
}

@InProceedings{diekmann2020dontpanic,
  title = {Don't Panic! Better, Fewer, Syntax Errors for LR Parsers},
  author = {Diekmann, Lukas and Tratt, Laurence},
  booktitle = {34th European Conference on Object-Oriented Programming (ECOOP 2020)},
  series = {Leibniz International Proceedings in Informatics (LIPIcs)},
  volume = {166},
  pages = {6:1--6:32},
  year = {2020},
  publisher = {Schloss Dagstuhl--Leibniz-Zentrum f{\"u}r Informatik},
  doi = {10.4230/LIPIcs.ECOOP.2020.6}
}

@book{puterman1994mdp,
  author    = {Martin L. Puterman},
  title     = {Markov Decision Processes: Discrete Stochastic Dynamic Programming},
  publisher = {John Wiley \& Sons},
  year      = {1994},
  series    = {Wiley Series in Probability and Statistics},
  doi       = {10.1002/9780470316887},
  isbn      = {978-0-471-61977-2}
}

@book{sutton2018rl,
  author    = {Richard S. Sutton and Andrew G. Barto},
  title     = {Reinforcement Learning: An Introduction},
  edition   = {2},
  publisher = {The MIT Press},
  year      = {2018},
  doi       = {10.5555/3312046}
}

@article{ramalho2024testing,
  title={Testing and Debugging Quantum Programs: The Road to 2030},
  author={Ramalho, Neilson Carlos Leite and de Souza, Higor Amario and Chaim, Marcos Lordello},
  journal={arXiv preprint arXiv:2405.09178},
  year={2024},
  doi={10.48550/arXiv.2405.09178}
}

@misc{Cuccaro2004ripple,
  author = {Cuccaro, Steven A. and Draper, Thomas G. and Kutin, Samuel A. and Moulton, David Petrie},
  title = {A new quantum ripple-carry addition circuit},
  doi = {10.48550/arxiv.quant-ph/0410184},
  publisher = {arXiv},
  year = {2004}
}

@article{kambhampati2024llms,
  title={{LLMs} can't plan, but can help planning in {LLM}-modulo frameworks},
  author={Kambhampati, Subbarao and Valmeekam, Karthik and Guan, Lin and Verma, Mudit and Stechly, Kaya and Bhambri, Siddhant and Saldyt, Lucas and Murthy, Anil},
  journal={arXiv preprint arXiv:2402.01817},
  year={2024}
}

@article{bennett1973logical,
  author  = {Bennett, Charles H.},
  title   = {Logical Reversibility of Computation},
  journal = {IBM Journal of Research and Development},
  volume  = {17},
  number  = {6},
  pages   = {525--532},
  year    = {1973},
  month   = nov,
  doi     = {10.1147/rd.176.0525}
}

@article{gulati2024evaluation,
  title={An Evaluation Benchmark for Autoformalization in {Lean4}},
  author={Gulati, Aryan and Ladsaria, Devanshu and Mishra, Shubhra and Sidhu, Jasdeep and Miranda, Brando},
  journal={arXiv preprint arXiv:2406.06555},
  year={2024}
}

@article{ospanov2025apollo,
  title={APOLLO: Automated {LLM} and {Lean} Collaboration for Advanced Formal Reasoning},
  author={Ospanov, Azim and Farnia, Farzan and Yousefzadeh, Roozbeh},
  journal={arXiv preprint arXiv:2505.05758},
  year={2025}
}

@article{cummins2023llmcompileropt,
  title={Large Language Models for Compiler Optimization},
  author={Cummins, Chris and Seeker, Volker and Grubisic, Dejan and Elhoushi, Mostafa and Liang, Youwei and Roziere, Baptiste and Gehring, Jonas and Gloeckle, Fabian and Hazelwood, Kim and Synnaeve, Gabriel and Leather, Hugh},
  journal={arXiv preprint arXiv:2309.07062},
  year={2023},
  doi={10.48550/arXiv.2309.07062}
}

@article{fang2024transferable,
  title={Transferable Presynthesis PPA Estimation for {RTL} Designs With Data Augmentation Techniques},
  author={Fang, Wenji and Lu, Yao and Liu, Shang and Zhang, Qijun and Xu, Ceyu and Wills, Lisa Wu and Zhang, Hongce and Xie, Zhiyao},
  journal={IEEE Transactions on Computer-Aided Design of Integrated Circuits and Systems},
  volume={44},
  number={1},
  pages={200--213},
  year={2024},
  publisher={IEEE}
}

@article{wei2022emergent,
  title={Emergent abilities of large language models},
  author={Wei, Jason and Tay, Yi and Bommasani, Rishi and Raffel, Colin and Zoph, Barret and Borgeaud, Sebastian and Yogatama, Dani and Bosma, Maarten and Zhou, Denny and Metzler, Donald and others},
  journal={TMLR},
  year={2022},
  doi={10.48550/arXiv.2206.07682}
}

@article{mckenzie2023inverse,
  title={Inverse scaling: When bigger isn't better},
  author={McKenzie, Ian R and Lyzhov, Alexander and Pieler, Michael and Parrish, Alicia and Mueller, Aaron and Prabhu, Ameya and McLean, Euan and Kirtland, Aaron and Ross, Alexis and Liu, Alisa and others},
  journal={TMLR},
  year={2023},
  doi={10.48550/arXiv.2306.09479}
}

@article{kim2023evidence,
  title={Evidence for the utility of quantum computing before fault tolerance},
  author={Kim, Youngseok and Eddins, Andrew and Anand, Sajant and Wei, Ken Xuan and Van Den Berg, Ewout and Rosenblatt, Sami and Nayfeh, Hasan and Wu, Yantao and Zaletel, Michael and Temme, Kristan and others},
  journal={Nature},
  volume={618},
  number={7965},
  pages={500--505},
  year={2023},
  publisher={Nature Publishing Group UK London},
  doi={10.1038/s41586-023-06096-3}
}

@article{li2022competition,
  title={Competition-level code generation with {AlphaCode}},
  author={Li, Yujia and Choi, David and Chung, Junyoung and Kushman, Nate and Schrittwieser, Julian and Leblond, R{\'e}mi and Eccles, Tom and Keeling, James and Gimeno, Felix and Dal Lago, Agustin and others},
  journal={Science},
  volume={378},
  number={6624},
  pages={1092--1097},
  year={2022},
  publisher={American Association for the Advancement of Science}
}

@article{dalzell2020howmany,
  title={How many qubits are needed for quantum computational supremacy?},
  author={Dalzell, Alexander M and Harrow, Aram W and Koh, Dax Enshan and La Placa, Rolando L},
  journal={Quantum},
  volume={4},
  pages={264},
  year={2020},
  doi={10.22331/q-2020-05-11-264}
}

@article{arute2019quantum,
  title={Quantum supremacy using a programmable superconducting processor},
  author={Arute, Frank and Arya, Kunal and Babbush, Ryan and Bacon, Dave and Bardin, Joseph C and Barends, Rami and Biswas, Rupak and Boixo, Sergio and Brandao, Fernando GSL and Buell, David A and others},
  journal={Nature},
  volume={574},
  number={7779},
  pages={505--510},
  year={2019},
  doi={10.1038/s41586-019-1666-5}
}

@book{nielsen2010quantum,
  title={Quantum Computation and Quantum Information: 10th Anniversary Edition},
  author={Nielsen, Michael A and Chuang, Isaac L},
  year={2010},
  publisher={Cambridge University Press},
  doi={10.1017/CBO9780511976667}
}

@article{hoffmann2022training,
  title={Training compute-optimal large language models},
  author={Hoffmann, Jordan and Borgeaud, Sebastian and Mensch, Arthur and Buchatskaya, Elena and Cai, Trevor and Rutherford, Eliza and Casas, Diego de Las and Hendricks, Lisa Anne and Welbl, Johannes and Clark, Aidan and others},
  journal={Advances in Neural Information Processing Systems},
  volume={35},
  pages={30016--30030},
  year={2022},
  doi={10.48550/arXiv.2203.15556}
}

@article{watrous2009quantum,
  title={Quantum computational complexity},
  author={Watrous, John},
  journal={Encyclopedia of Complexity and Systems Science},
  pages={7174--7201},
  year={2009},
  publisher={Springer},
  doi={10.1007/978-0-387-30440-3_428}
}

@article{mele2024introduction,
  title={Introduction to {Haar} Measure Tools in Quantum Information: A Beginner's Tutorial},
  author={Mele, Antonio Anna},
  journal={Quantum},
  volume={8},
  pages={1340},
  year={2024},
  doi={10.22331/q-2024-05-08-1340}
}

@article{schuch2008entropy,
  title={Entropy scaling and simulability by matrix product states},
  author={Schuch, Norbert and Wolf, Michael M and Verstraete, Frank and Cirac, J Ignacio},
  journal={Physical Review Letters},
  volume={100},
  number={3},
  pages={030504},
  year={2008},
  doi={10.1103/PhysRevLett.100.030504}
}

@article{vidal2003efficient,
  title={Efficient classical simulation of slightly entangled quantum computations},
  author={Vidal, Guifr{\'e}},
  journal={Physical Review Letters},
  volume={91},
  number={14},
  pages={147902},
  year={2003},
  doi={10.1103/PhysRevLett.91.147902}
}

@inproceedings{shin2021constrained,
  title={Constrained Language Models Yield Few-Shot Semantic Parsers},
  author={Shin, Richard and Lin, Christopher H and Thomson, Sam and Chen, Charles and Roy, Subhro and Platanios, Emmanouil Antonios and Pauls, Adam and Klein, Dan and Eisner, Jason and Van Durme, Benjamin},
  booktitle={Proceedings of the 2021 Conference on Empirical Methods in Natural Language Processing},
  pages={7699--7715},
  year={2021},
  doi={10.18653/v1/2021.emnlp-main.608}
}

@article{lightman2023lets,
  title={Let's verify step by step},
  author={Lightman, Hunter and Kosaraju, Vineet and Burda, Yura and Edwards, Harri and Baker, Bowen and Lee, Teddy and Leike, Jan and Schulman, John and Sutskever, Ilya and Cobbe, Karl},
  journal={arXiv preprint arXiv:2305.20050},
  year={2023},
  doi={10.48550/arXiv.2305.20050}
}

@article{zhou2024larger,
  title={Larger and more instructable language models become less reliable},
  author={Zhou, Lexin and Schellaert, Wout and Mart{\'\i}nez-Plumed, Fernando and Moros-Daval, Yael and Ferri, C{\`e}sar and Hern{\'a}ndez-Orallo, Jos{\'e}},
  journal={Nature},
  volume={634},
  number={8032},
  pages={61--68},
  year={2024},
  doi={10.1038/s41586-024-07930-y}
}

@article{dziri2023faith,
  title={Faith and fate: Limits of transformers on compositionality},
  author={Dziri, Nouha and Lu, Ximing and Sclar, Melanie and Li, Xiang Lorraine and Jiang, Liwei and Lin, Bill Yuchen and Welleck, Sean and West, Peter and Bhagavatula, Chandra and Le Bras, Ronan and others},
  journal={Advances in Neural Information Processing Systems},
  volume={36},
  pages={70293--70332},
  year={2023},
  doi={10.48550/arXiv.2305.18654}
}

@article{bouland2018quantum,
  title={Quantum supremacy and the complexity of random circuit sampling},
  author={Bouland, Adam and Fefferman, Bill and Nirkhe, Chinmay and Vazirani, Umesh},
  journal={arXiv preprint arXiv:1803.04402},
  year={2018}
}

@article{tsai2025beyond,
  title={Beyond Natural Language Perplexity: Detecting Dead Code Poisoning in Code Generation Datasets},
  author={Tsai, Chi-Chien and Yu, Chia-Mu and Lin, Ying-Dar and Wu, Yu-Sung and Lee, Wei-Bin},
  journal={arXiv preprint arXiv:2502.20246},
  year={2025}
}

@article{javadi2024quantum,
  title={Quantum computing with {Qiskit}},
  author={Javadi-Abhari, Ali and Treinish, Matthew and Krsulich, Kevin and Wood, Christopher J and Lishman, Jake and Gacon, Julien and Martiel, Simon and Nation, Paul D and Bishop, Lev S and Cross, Andrew W and others},
  journal={arXiv preprint arXiv:2405.08810},
  year={2024}
}

@article{huang2020masking,
   title={A Closer Look at Invalid Action Masking in Policy Gradient Algorithms},
   volume={35},
   ISSN={2334-0762},
   DOI={10.32473/flairs.v35i.130584},
   journal={The International FLAIRS Conference Proceedings},
   publisher={University of Florida George A Smathers Libraries},
   author={Huang, Shengyi and Ontañón, Santiago},
   year={2022},
   month=may 
}

@article{hou2023masking,
  author    = {Yueqi Hou and Xiaolong Liang and Jiaqiang Zhang and Qisong Yang and Aiwu Yang and Ning Wang},
  title     = {Exploring the Use of Invalid Action Masking in Reinforcement Learning: A Comparative Study of On-Policy and Off-Policy Algorithms in Real-Time Strategy Games},
  journal   = {Applied Sciences},
  year      = {2023},
  volume    = {13},
  number    = {14},
  pages     = {8283},
  doi       = {10.3390/app13148283}
}

@inproceedings{miceli2023larger,
  title={The larger they are, the harder they fail: Language models do not recognize identifier swaps in python},
  author={Miceli-Barone, Antonio Valerio and Barez, Fazl and Cohen, Shay B and Konstas, Ioannis},
  booktitle={Findings of the Association for Computational Linguistics: ACL 2023},
  pages={272--292},
  year={2023}
}

@article{litinski2019game,
  title={A game of surface codes: Large-scale quantum computing with lattice surgery},
  author={Litinski, Daniel},
  journal={Quantum},
  volume={3},
  pages={128},
  year={2019},
  publisher={Verein zur F{\"o}rderung des Open Access Publizierens in den Quantenwissenschaften}
}

@article{fowler2018low,
  title={Low overhead quantum computation using lattice surgery},
  author={Fowler, Austin G and Gidney, Craig},
  journal={arXiv preprint arXiv:1808.06709},
  year={2018}
}

@article{beverland2022assessing,
  title={Assessing requirements to scale to practical quantum advantage},
  author={Beverland, Michael E and Murali, Prakash and Troyer, Matthias and Svore, Krysta M and Hoefler, Torsten and Kliuchnikov, Vadym and Low, Guang Hao and Soeken, Mathias and Sundaram, Aarthi and Vaschillo, Alexander},
  journal={arXiv preprint arXiv:2211.07629},
  year={2022}
}

@inproceedings{siraichi2018qubit,
  title={Qubit allocation},
  author={Siraichi, Marcos Yukio and Santos, Vin{\'\i}cius Fernandes dos and Collange, Caroline and Pereira, Fernando Magno Quint{\~a}o},
  booktitle={Proceedings of the 2018 international symposium on code generation and optimization},
  pages={113--125},
  year={2018}
}

@inproceedings{botea2018complexity,
  title={On the complexity of quantum circuit compilation},
  author={Botea, Adi and Kishimoto, Akihiro and Marinescu, Radu},
  booktitle={Proceedings of the International Symposium on Combinatorial Search},
  volume={9},
  pages={138--142},
  year={2018}
}

@inproceedings{tan2024sat,
  title={A sat scalpel for lattice surgery: representation and synthesis of subroutines for surface-code fault-tolerant quantum computing},
  author={Tan, Daniel Bochen and Niu, Murphy Yuezhen and Gidney, Craig},
  booktitle={2024 ACM/IEEE 51st Annual International Symposium on Computer Architecture (ISCA)},
  pages={325--339},
  year={2024},
  organization={IEEE}
}

@inproceedings{ito2023algorithmic,
  title={Algorithmic theory of qubit routing},
  author={Ito, Takehiro and Kakimura, Naonori and Kamiyama, Naoyuki and Kobayashi, Yusuke and Okamoto, Yoshio},
  booktitle={Algorithms and Data Structures Symposium},
  pages={533--546},
  year={2023},
  organization={Springer}
}

@article{silva2024multi,
  title={Multi-qubit lattice surgery scheduling},
  author={Silva, Allyson and Zhang, Xiangyi and Webb, Zak and Kramer, Mia and Yang, Chan Woo and Liu, Xiao and Lemieux, Jessica and Chen, Ka-Wai and Scherer, Artur and Ronagh, Pooya},
  journal={arXiv preprint arXiv:2405.17688},
  year={2024}
}

@book{federer1969geometric,
  author    = {Federer, Herbert},
  title     = {Geometric Measure Theory},
  publisher = {Springer},
  year      = {1969},
  series    = {Grundlehren der mathematischen Wissenschaften},
  volume    = {153},
}

@article{dawson2005solovay,
  author  = {Dawson, Christopher M. and Nielsen, Michael A.},
  title   = {The {Solovay-Kitaev} theorem},
  journal = {Quantum Information \& Computation},
  volume  = {6},
  number  = {1},
  pages   = {81--95},
  year    = {2006},
  eprint  = {quant-ph/0505030},
}

@article{gema2025inverse,
  title={Inverse Scaling in Test-Time Compute},
  author={Gema, Aryo Pradipta and H{\"a}gele, Alexander and Chen, Runjin and Arditi, Andy and Goldman-Wetzler, Jacob and Fraser-Taliente, Kit and Sleight, Henry and Petrini, Linda and Michael, Julian and Alex, Beatrice and others},
  journal={Transactions on Machine Learning Research},
  year={2025}
}

@article{song2026data,
  title={Data Verification is the Future of Quantum Computing Copilots},
  author={Song, Junhao and Bi, Ziqian and Chia, Xinliang and Knottenbelt, William and Cao, Yudong},
  journal={In 40th Proceedings of the Association for the Advancement of Artificial Intelligence AAAI 2026 AI4Research Workshop},
  year={2026}
}

@inproceedings{song2024sais,
  title={{SAIS}: A novel bio-inspired artificial immune system based on symbiotic paradigm},
  author={Song, Junhao and Yuan, Yingfang and Pang, Wei},
  booktitle={Proceedings of the Genetic and Evolutionary Computation Conference Companion},
  pages={2115--2118},
  year={2024}
}

@article{mcclean2018barren,
  author  = {McClean, Jarrod R. and Boixo, Sergio and Smelyanskiy, Vadim N. 
             and Babbush, Ryan and Neven, Hartmut},
  title   = {Barren plateaus in quantum neural network training landscapes},
  journal = {Nature Communications},
  volume  = {9},
  number  = {1},
  pages   = {4812},
  year    = {2018},
  doi     = {10.1038/s41467-018-07090-4},
}

@article{cerezo2021cost,
  author  = {Cerezo, M. and Sone, Akira and Volkoff, Tyler 
             and Cincio, Lukasz and Coles, Patrick J.},
  title   = {Cost function dependent barren plateaus in shallow 
             parameterized quantum circuits},
  journal = {Nature Communications},
  volume  = {12},
  number  = {1},
  pages   = {1791},
  year    = {2021},
  doi     = {10.1038/s41467-021-21728-w},
}

@article{barke2023grounded,
  author = {Barke, Shraddha and James, Michael B. and Polikarpova, Nadia},
  title = {Grounded {Copilot}: {H}ow Programmers Interact with Code-Generating Models},
  journal = {Proceedings of the ACM on Programming Languages},
  volume = {7},
  number = {OOPSLA1},
  pages = {85--111},
  year = {2023},
  doi = {10.1145/3586030}
}

@inproceedings{vaithilingam2022expectation,
  author = {Vaithilingam, Priyan and Zhang, Tianyi and Glassman, Elena L.},
  title = {Expectation vs. Experience: Evaluating the Usability of Code Generation Tools Powered by Large Language Models},
  booktitle = {Extended Abstracts of the 2022 CHI Conference on Human Factors in Computing Systems},
  series = {CHI EA '22},
  year = {2022},
  pages = {1--7},
  doi = {10.1145/3491101.3519665}
}

@inproceedings{liang2024understanding,
  author = {Liang, Jenny T. and Yang, Chenyang and Myers, Brad A.},
  title = {A Large-Scale Survey on the Usability of {AI} Programming Assistants: {S}uccesses and Challenges},
  booktitle = {Proceedings of the 46th IEEE/ACM International Conference on Software Engineering (ICSE)},
  year = {2024}
}

@inproceedings{scholak2021picard,
  title={{PICARD}: Parsing Incrementally for Constrained Auto-Regressive Decoding from Language Models},
  author={Scholak, Torsten and Schucher, Nathan and Bahdanau, Dzmitry},
  booktitle={Proceedings of the 2021 Conference on Empirical Methods in Natural Language Processing},
  pages={9895--9901},
  year={2021},
  publisher={Association for Computational Linguistics}
}

@inproceedings{mahadev2018classical,
  title={Classical verification of quantum computations},
  author={Mahadev, Urmila},
  booktitle={2018 IEEE 59th Annual Symposium on Foundations of Computer Science (FOCS)},
  pages={259--267},
  year={2018},
  organization={IEEE},
  doi={10.1109/FOCS.2018.00033}
}

@article{fitzsimons2017unconditionally,
  title={Unconditionally verifiable blind quantum computation},
  author={Fitzsimons, Joseph F. and Kashefi, Elham},
  journal={Physical Review A},
  volume={96},
  number={1},
  pages={012303},
  year={2017},
  doi={10.1103/PhysRevA.96.012303}
}

@article{reichardt2013classical,
  title={Classical command of quantum systems},
  author={Reichardt, Ben W. and Unger, Falk and Vazirani, Umesh},
  journal={Nature},
  volume={496},
  number={7446},
  pages={456--460},
  year={2013},
  doi={10.1038/nature12035}
}

@inproceedings{bartusek2022succinct,
  title={Succinct classical verification of quantum computation},
  author={Bartusek, James and Kalai, Yael Tauman and Lombardi, Alex and Ma, Fermi and Malavolta, Giulio and Vaikuntanathan, Vinod and Vidick, Thomas and Yang, Lisa},
  booktitle={Annual International Cryptology Conference},
  pages={195--211},
  year={2022},
  organization={Springer},
  doi={10.1007/978-3-031-15979-4\_7}
}
\bibliographystyle{icml2026}

%%%%%%%%%%%%%%%%%%%%%%%%%%%%%%%%%%%%%%%%%%%%%%%%%%%%%%%%%%%%%%%%%%%%%%%%%%%%%%%
%                              APPENDIX
%%%%%%%%%%%%%%%%%%%%%%%%%%%%%%%%%%%%%%%%%%%%%%%%%%%%%%%%%%%%%%%%%%%%%%%%%%%%%%%
\newpage
\appendix
\onecolumn

{\Large \bf Appendices}

\section{The Fallacy of Local Modularity.}\label{ex:fallacy}
{
    Below we use an example to illustrate a fundamental difference between quantum and classical programming (also see Figure \ref{fig:fig2_comparison} for a summary).

    \textbf{The classical regime (local modularity):} Consider the classical bit program: \texttt{x = y;} \texttt{temp = x;} where \texttt{x}, \texttt{y}, and \texttt{temp} are bits. Classically, both lines are syntactically and semantically acceptable under routine inspection: the assignment \texttt{temp = x} may be unused without invalidating the program (at most producing a compiler warning). More broadly, classical semantics tolerates \emph{erasure}: overwriting \texttt{x} discards the old value of \texttt{x} with no global obligation to retain invertibility.

    \textbf{The quantum reality (global interference):} This analogy collapses in quantum computing because valid programs must extend to \emph{reversible} (unitary) transformations. Overwriting a register (e.g., \texttt{x := y}) is not a bijection unless the overwritten information is preserved elsewhere \cite{bennett1973logical} and any temporary workspace is later \emph{uncomputed} to decouple it from the computational register.
    A syntactically correct circuit can therefore be globally invalid in the sense of leaving residual entanglement that destroys later interference.

    \begin{itemize}[leftmargin=*, noitemsep, topsep=-2pt]
        \item \textit{Failure Mode (discarding a ``harmless'' temp):}
        A common reversible embedding of the classical intent is to compute into ancillas and then uncompute intermediate work.
        Suppose an LLM produces the following circuit fragment intended to mirror the classical program:
        \begin{flushleft}
        \ttfamily\small
        CNOT(y, x)\hspace{2.4em}\# ``x = y''\\
        CNOT(x, temp)\hspace{0.6em}\# ``temp = x''\\
        CNOT(y, x)\hspace{2.4em}\# uncompute ``x = y''
        \end{flushleft}
        The last line correctly restores \texttt{x} to its original value, but the fragment \emph{ignores} the fact that \texttt{temp} now carries information about \texttt{y}.
        If \texttt{y} is in a superposition $\frac{1}{\sqrt{2}}(\ket{0}+\ket{1})$ and \texttt{temp} starts at $\ket{0}$, the map produces the entangled state
        \[
            \frac{1}{\sqrt{2}}(\ket{0}_\texttt{y}\ket{0}_\texttt{temp} + \ket{1}_\texttt{y}\ket{1}_\texttt{temp}),
        \]
        so \texttt{temp} is \textbf{not} a benign leftover variable: it is a \emph{which-branch record}.
        Any subsequent algorithm that relies on interference in the $\texttt{y}$ register will generally fail because coherence has leaked into \texttt{temp}.

        \item \textit{Consequence (non-local semantic constraints):}
        In classical code, leaving an unused \texttt{temp} is harmless; in quantum circuits, leaving a ``temporary'' register entangled with the computational state {changes the reduced state} of the register and can {destroy the interference pattern} required for the computation.
        This is a \emph{global} semantic failure: no local syntax parser can detect it, and it is not repaired by superficial fluency.

        \item \textit{Implication for LLMs (why scale is not enough):}
        Next-token training optimizes for \emph{syntactic likelihood}, not for satisfaction of hard mathematical invariants such as reversibility/unitarity (and, on hardware, topology and scheduling constraints).
        The toy example already exhibits the core mismatch: correctness depends on a global bijection and explicit uncomputation, not on locally plausible lines.
        Scaling up LLMs can improve stylistic fidelity, but it does not by itself enforce the required algebraic constraints; reliable generation must instead couple decoding to explicit verifiers/solvers and uncomputation-aware compilation, i.e., a \emph{verification-centric} loop rather than ``generate-then-hope''.

        \item \textit{Efficient identification of the failure mode:}
        Unlike general quantum programs, which costs exponential time to simulate, failure modes such as this example can be identified without exponential-time simulation. Concretely, one can track the Pauli stabilizer $Z$ acting on qubit $y$ throughout the circuit, since the circuit consists entirely of Clifford gates and is therefore efficiently simulable via the Gottesman--Knill theorem. Ordering the qubits as $(y,x,\mathrm{temp})$, the stabilizer evolution is:
        \[
        ZII \;\xrightarrow{\;\mathrm{CNOT}(y,x)\;}\; ZZI
        \;\xrightarrow{\;\mathrm{CNOT}(x,\mathrm{temp})\;}\; ZZZ
        \;\xrightarrow{\;\mathrm{CNOT}(y,x)\;}\; ZIZ.
        \]
        If the intended computation had succeeded, the temporary register would be cleanly isolated at the end, and we would expect the final operator to be $IIZ$, indicating that only $\mathrm{temp}$ carries the residual information. Instead, the actual final operator is $ZIZ$, which shows that the output still has support on both $y$ and $\mathrm{temp}$.
    \end{itemize}
}
\newpage

\section{Structural Sampling Hardness in NISQ and FTQC}
\label{app:struct-sampling-hardness}

This appendix expands the structural-sampling claim of 
Section~\ref{sec:div_high}. We first formalize the verification 
cost and pass-rate decomposition for both regimes, then prove the 
worst-case sampling cost for FTQC.

\subsection{Verification Cost Decomposition}
\label{app:struct-verification}

For a candidate circuit description $x$ on $n$ logical qubits, depth $d$, 
and $G = \Theta(nd)$ operations, the cost of deciding $x \in 
\mathcal{M}_\text{struct}$ decomposes additively over the local 
constraints in $x$:
\begin{equation}
\label{eq:struct-decomp}
C_\text{struct}(x) \;=\; C_\text{parse}(x) \;+\; \sum_{k=1}^{K} C_k(x),
\end{equation}
where $C_\text{parse}(x) = \Theta(|x|) = \Theta(nd)$ is the cost of 
reading the circuit representation, often described in the language of some intermediate representation (IR), and $C_k(x)$ is the cost of the $k$-th constraint check.

\textbf{NISQ checks.} The relevant constraint classes are:

\begin{table}[h]
\centering
\caption{NISQ structural checks. $G$ is total gate count, $G_2 \leq G$ is the number of two-qubit gates.}
\label{tab:nisq_checks}
\begin{tabular}{lll}
\toprule
Check & Per-operation cost & Total \\
\midrule
Gate-set membership & $\mathcal{O}(1)$ table lookup & $\Theta(G)$ \\
Coupling-graph adjacency & $\mathcal{O}(1)$ hash lookup & $\Theta(G_2)$ \\
Wire well-formedness & $\mathcal{O}(1)$ & $\Theta(G)$ \\
Coherence-time budget (optional) & $\mathcal{O}(1)$ & $\Theta(G)$ \\
\bottomrule
\end{tabular}
\end{table}

\begin{table}[h]
\centering
\caption{FTQC structural checks. $F$ is the number of magic-state factories, $d$ is circuit depth.}
\label{tab:ftqc_checks}
\begin{tabular}{lll}
\toprule
Check & Per-op or per-step cost & Total \\
\midrule
Patch-layout adjacency & $\mathcal{O}(1)$ & $\Theta(G)$ \\
Boundary-type matching & $\mathcal{O}(1)$ & $\Theta(G)$ \\
Code-distance budget per op & $\mathcal{O}(1)$ arithmetic & $\Theta(G)$ \\
T-throughput per timestep & $\mathcal{O}(F)$ & $\Theta(Fd)$ \\
Routing-channel availability & $\mathcal{O}(1)$ per surgery & $\Theta(G)$ \\
Pauli-frame consistency & $\mathcal{O}(1)$ per op & $\Theta(G)$ \\
\bottomrule
\end{tabular}
\end{table}

\noindent where $F$ is the number of magic-state factories. Under the 
standard scaling $F = \mathcal{O}(n)$ \citep{beverland2022assessing}, the 
total is $\Theta(nd + Fd) = \Theta(nd)$. Crucially, all checks are local 
in $x$: each references a constant number of operations or a single 
timestep slice. None require simulating the circuit or computing 
information not already present in $x$.

We treat the layout, schedule, and Pauli-frame annotations as part of the 
candidate description $x$, separating verification from the upstream 
synthesis problem of \emph{producing} such annotations from an unannotated 
logical circuit. The latter is NP-hard in both regimes 
\citep{siraichi2018qubit,ito2023algorithmic,tan2024sat} and is upstream 
of $\mathcal{M}_\text{struct}$ membership.

\subsection{NISQ Pass Rate}
\label{app:nisq-pass}

In NISQ, $\mathcal{M}_\text{struct}$ is a conjunction of independent local 
constraints, one per gate (more precisely, per two-qubit gate for the 
coupling-graph constraint, and per gate for the gate-set and wire 
constraints). If $\epsilon$ denotes the per-gate violation probability of 
a model fine-tuned on coupling-graph-aware data, the structural pass rate 
is
\begin{equation}
\label{eq:nisq-pass-app}
P(x \in \mathcal{M}_\text{struct}^\text{NISQ}) 
\;=\; \prod_{g \in x} (1 - \epsilon_g) 
\;\geq\; (1 - \epsilon)^G,
\end{equation}
where $\epsilon = \max_g \epsilon_g$. For $\epsilon = o(1/G)$, this 
converges to $\Theta(1)$ as $G \to \infty$. Empirically, NISQ-targeted 
code-generation systems such as Qiskit Code Assistant achieve near-unity 
structural pass rates after task-specific fine-tuning 
\citep{dupuis2024qiskit, vishwakarma2024qiskit}, consistent with 
the theoretical scaling.

\subsection{FTQC Sampling Hardness}
\label{app:ftqc-sampling-hardness}

The FTQC structural constraints are not separable per-gate: a circuit may 
be locally legal at every operation yet globally infeasible because, for 
instance, three lattice surgeries demand the same routing channel at the 
same timestep, or the cumulative T-gate demand exceeds factory throughput 
at some peak load. This non-separability is what gives rise to the NP-hardness 
of the underlying feasibility problems.

\begin{proposition}[Structural sampling hardness in FTQC]
\label{prop:ftqc-sampling-app}
There exists a polynomial-time-constructible family of FTQC compilation 
specifications $\{\Sigma_n\}_{n \geq 1}$, on $n$ logical qubits and depth 
$d_n = \mathrm{poly}(n)$, with the following property. Let $\pi_\theta$ 
be any distribution over candidate circuit descriptions that can be sampled 
in time $\mathrm{poly}(n)$. Then either
\begin{equation}
\label{eq:ftqc-sampling-bound}
P_{x \sim \pi_\theta}\!\left(x \in \mathcal{M}_\text{struct}^\text{FTQC}(\Sigma_n)\right) 
\;\leq\; e^{-\Omega(n)} \qquad \text{for sufficiently large } n,
\end{equation}
or $\mathsf{NP} \subseteq \mathsf{BPP}$.
\end{proposition}

\begin{proof}
We reduce 3-SAT to the FTQC structural-feasibility problem and combine 
with standard BPP-amplification.

\textbf{Reduction.} Let $\phi$ be a 3-SAT instance on $m$ variables and 
$\mathrm{poly}(m)$ clauses. Qubit allocation: given a logical quantum 
circuit and a hardware coupling graph, decide whether a connectivity-respecting 
allocation exists, which is NP-complete via a polynomial-time reduction from 
subgraph isomorphism \citep{siraichi2018qubit, botea2018complexity}. Composing 
this with the standard linear-blowup reduction from 3-SAT through 
constraint-graph encoding yields, in $\mathrm{poly}(m)$ time, an FTQC 
compilation specification $\Sigma(\phi)$ on $n = \Theta(m)$ logical qubits 
and depth $d = \mathrm{poly}(m)$, such that the set of descriptions $x$ 
satisfying $x \in \mathcal{M}_\text{struct}^\text{FTQC}(\Sigma(\phi))$ is in 
$\mathrm{poly}(m)$-time computable bijection with the set of satisfying 
assignments of $\phi$. The surface-code lattice-surgery substrate inherits 
this hardness because structurally valid descriptions must include a 
patch-to-tile assignment respecting the lattice's adjacency structure 
\citep{litinski2019game, tan2024sat}; the SAT encoding of 
\citet{tan2024sat} establishes that this subproblem is at least as hard 
as general Boolean satisfaction in practice.

\textbf{Amplification.} Suppose for contradiction that there exists a 
$\mathrm{poly}(n)$-time-samplable $\pi_\theta$ satisfying 
$P(x \in \mathcal{M}_\text{struct}^\text{FTQC}(\Sigma(\phi))) \geq 2^{-cn}$ 
for some constant $c$ and all sufficiently large $n$. Then for $\phi$ on 
$m$ variables, repeated independent sampling from $\pi_\theta$ followed by 
$\Theta(nd) = \mathrm{poly}(m)$-time structural verification 
(Section~\ref{app:struct-verification}) yields a satisfying assignment of 
$\phi$ in expected $2^{cn} = 2^{\Theta(m)}$ trials. For $c < 1$, this is 
$o(2^m)$ expected work, contradicting the Exponential Time Hypothesis; for the unconditional statement that 
$\mathsf{NP} \not\subseteq \mathsf{BPP}$, take $c$ any constant and amplify 
to constant success probability via $\mathrm{poly}(n)$ independent repetitions, 
which contradicts $\mathsf{NP} \not\subseteq \mathsf{BPP}$.
\end{proof}

\textbf{Remarks.}

\textbf{(i) The proposition is a worst-case statement over compilation 
specifications.} Practical quantum programs have regular structure 
(repeated QFT blocks, regular Trotter steps, structured amplitude 
estimation) that admits polynomial-time schedule construction 
\citep{litinski2019game, beverland2022assessing}. The 
verifier-centric paradigm of Section \ref{sec:verifier} exploits 
this structure constructively.

\textbf{(ii) The hardness is structural, not quantum.} The same argument applies 
to any NP-hard discrete satisfaction problem embedded into a sampling 
target. What is quantum-specific is the natural emergence of such 
constraints from physical considerations (lattice surgery, magic-state 
distillation, error-correction thresholds), making FTQC a setting where 
the structural barrier is unavoidable rather than artificial.

\textbf{(iii) The exponent $\Omega(n)$ in Eq.~\eqref{eq:ftqc-sampling-bound} 
depends on the qubit-blowup of the underlying reduction.} The construction 
sketched in the proof uses qubit allocation NP-hardness 
\citep{siraichi2018qubit, botea2018complexity} composed with the surface-code 
substrate; tighter encodings (e.g., direct SAT encodings of lattice-surgery 
scheduling \citep{tan2024sat}) may yield sharper constants. The precise 
exponent is not load-bearing for the argument: any exponential decay 
suffices to make worst-case FTQC post-selection exponentially worse than 
the NISQ case (Section~\ref{sec:exponential-cost}).

% \subsubsection{Implication for expected sampling cost}
% \label{app:struct-expected-cost}

% Combining Eqs.~\eqref{eq:struct-decomp}, \eqref{eq:nisq-pass-app}, and 
% Proposition~\ref{prop:ftqc-sampling-app}, the expected cost of producing 
% \emph{one} structurally valid candidate is
% \begin{equation}
% \label{eq:struct-expected-cost-app}
% \mathbb{E}\!\left[C \text{ to obtain one } x \in \mathcal{M}_\text{struct}\right] 
% \;=\; 
% \frac{C_\text{struct}}{P(\mathcal{M}_\text{struct})}
% \;=\;
% \begin{cases}
% \Theta(nd) & \text{(NISQ)}, \\[6pt]
% e^{\Omega(n)} \cdot \Theta(nd) & \text{(FTQC, worst case)}.
% \end{cases}
% \end{equation}

% In NISQ, structural filtering is essentially free; in FTQC, it contributes 
% its own exponential factor to the post-selection cost, \emph{before} 
% functional-validity rejection (Section~\ref{sec:div_high}) is 
% considered. The total expected cost in FTQC is then the product of the 
% two pass-rate denominators, giving the coupled exponential analyzed in 
% Section~\ref{sec:exponential-cost}.

\section{Functional Validity Analysis}
\label{app:func-validity}

This appendix expands Section~\ref{sec:func-validity}. 
Appendix~\ref{app:func-verification} discusses the verification cost in 
detail, including specialized cases where the $\Theta(2^n)$ bound can be 
beaten. Appendix~\ref{app:func-sparsity} states and proves the formal 
sparsity bound underlying Equation~\eqref{eq:func-pass}.

\subsection{Verification Cost $C_\text{func}$}
\label{app:func-verification}

\textbf{Exact verification.} Given a candidate circuit $x$ on $n$ 
qubits at depth $d$ and a target unitary $U_\text{target} \in SU(2^n)$, 
the most direct verification scheme is exact state-vector simulation: 
maintain a length-$2^n$ complex vector, apply each of the $\Theta(nd)$ 
gates as a sparse linear operator, and compare the final state against 
$U_\text{target}|\psi\rangle$ for a basis of input states $|\psi\rangle$. 
The per-gate cost is $\Theta(2^n)$ for generic two-qubit gates, giving 
total cost $\Theta(nd \cdot 2^n)$. Memory is $\Theta(2^n)$. Both scale 
exponentially in $n$, becoming infeasible beyond approximately 40--50 
qubits on current classical hardware \citep{dalzell2020howmany,arute2019quantum}.

\textbf{Approximate decision is also hard.} A weaker verification 
target: merely deciding whether $\|U(x) - U_\text{target}\| \leq 
\varepsilon$ versus $\|U(x) - U_\text{target}\| \geq \varepsilon'$ for 
some gap $\varepsilon' > \varepsilon$, does not escape the exponential 
barrier. The complementary problem (non-identity check) is 
QMA-complete \citep{watrous2009quantum}: a polynomial-time classical algorithm 
deciding non-identity for arbitrary polynomial-depth circuits would 
imply $\mathsf{QMA} \subseteq \mathsf{P}$. Even with a quantum verifier 
and polynomial quantum advice, the problem is not known to be in 
$\mathsf{P}$.

\textbf{Tensor networks fail in the relevant regime.} Matrix product 
state (MPS) and more general tensor-network simulators reduce verification 
cost to $\mathrm{poly}(n, \chi)$, where $\chi$ is the bond dimension 
required to faithfully represent the intermediate states. For circuits 
whose intermediate states have entanglement entropy scaling as the 
boundary of a subregion (area law), $\chi$ is bounded by a constant or 
polynomial in $n$, and verification is efficient. However, useful 
post-classical quantum algorithms generate volume-law entanglement: the 
entropy of a contiguous subregion of $k$ qubits scales as $\Theta(k)$ 
rather than $\Theta(\partial k)$ \citep{schuch2008entropy}. Volume-law 
entanglement forces $\chi = e^{\Omega(n)}$ \citep{vidal2003efficient}, returning 
the simulation cost to $\Theta(2^n)$. This is not a limitation of MPS 
specifically: any tensor network whose contraction cost is polynomial in 
$\chi$ inherits the same exponential blowup for volume-law states.

\textbf{Specialized cases that escape the exponential.} For 
completeness, we note three regimes where $C_\text{func}$ is polynomial:

\begin{itemize}
\item \emph{Clifford fragments.} Circuits composed entirely of Clifford 
gates (CNOT, H, S) admit polynomial-time stabilizer simulation by the 
Gottesman--Knill theorem. Verification of Clifford subcircuits against 
Clifford targets is $\mathrm{poly}(n, d)$. This is exploited in our 
Appendix Example~1 to identify the uncomputation failure mode without 
exponential simulation.

\item \emph{Structured Hamiltonian targets.} For Hamiltonian simulation 
of local Hamiltonians on geometrically structured lattices, the target 
$U_\text{target} = e^{-iHt}$ has known polynomial-time verifiers via 
local observables, provided the simulation time $t$ is short enough that 
the Lieb--Robinson cone remains local.

\item \emph{ZX-calculus normalization.} For Clifford+T circuits, the 
ZX-calculus admits polynomial-time normalization to canonical form, 
enabling functional equivalence checking in $\mathrm{poly}(n, d)$ time 
for the Clifford sector and exponential time only for the T-count.
\end{itemize}

These regimes do not invalidate the $\Theta(2^n)$ bound: they describe 
\emph{structured} subspaces of $\mathcal{M}_\text{func}$, and the 
generic synthesis problem this paper addresses --- arbitrary 
$U_\text{target} \in SU(2^n)$ without prior structural knowledge --- 
remains exponential. Indeed, the existence of polynomial-time verifiers 
for structured cases is consistent with our broader thesis 
(Section~\ref{sec:verifier}): exploitable structure must be 
encoded \emph{constructively} into generation, not discovered post hoc.

\subsection{The Semantic Entropy Gap}
\label{app:func-sparsity}

We formalize the bound stated in Equation~\eqref{eq:func-pass} in two 
settings that together cover the practical program generation pipeline: 
the continuous parameter manifold (Appendix~\ref{app:func-sparsity-nisq}, 
natural for variational NISQ algorithms with parameterized rotation 
gates) and the discrete program description (Appendix~\ref{app:func-sparsity-ftqc}, 
natural for FTQC and for any high-level program representation operating 
over a finite instruction set). The two settings differ in mathematical 
machinery: Lipschitz volume counting vs.\ combinatorial enumeration, but yield the same qualitative bound.

\textbf{Common setup.} In both settings, $U_\text{target} \in SU(2^n)$ 
is the target unitary, $\pi_\theta$ is the model's distribution 
conditioned on $\mathcal{M}_\text{struct}$, and $f$ denotes the 
representation-to-unitary map. The bound is stated for a pointwise 
$U_\text{target}$ (no Haar or random-distribution assumption). The 
non-pathology constant $C_\theta$ measures how far $\pi_\theta$ deviates 
from the uniform reference distribution on its support; the assumption 
$C_\theta = \mathcal{O}(\mathrm{poly}(n))$ expresses that the model has 
not memorized the specific solution.

\subsubsection{NISQ Analysis: Continuous Parameter Manifold}
\label{app:func-sparsity-nisq}

\textbf{Setup.} Fix circuit depth $d$ and let $\mathcal{A}_d^\text{NISQ}$ 
denote the parameter manifold of structurally valid $n$-qubit ansatzes 
at depth $d$ with continuous rotation parameters. Each parameterized 
gate contributes $\mathcal{O}(1)$ continuous parameters and the ansatz 
contains $\Theta(nd)$ gates, so $\dim_\mathbb{R} \mathcal{A}_d^\text{NISQ} = D 
\leq c_0 nd$ for a constant $c_0$ independent of $n$ 
\citep{nielsen2010quantum}. Compilation defines a smooth map 
$f^\text{NISQ}: \mathcal{A}_d^\text{NISQ} \to SU(2^n)$. We bound the 
conditional pass rate by the Radon--Nikodym derivative against the uniform 
base measure $\mu$:
\begin{equation}
\label{eq:pass-bound-nisq}
P_{x \sim \pi_\theta}\!\left(\|U(x) - U_\text{target}\| \leq \varepsilon\right)
\leq C_\theta \cdot 
\frac{\mu\!\left((f^\text{NISQ})^{-1}(B_\varepsilon(U_\text{target}))\right)}
     {\mu(\mathcal{A}_d^\text{NISQ})},
\end{equation}
where $C_\theta := \|d\pi_\theta / d\mu\|_\infty$.

\begin{proposition}[Semantic entropy gap, NISQ]
\label{prop:func-sparsity-nisq}
Under the setup above, for any $U_\text{target} \in SU(2^n)$ and any 
$\varepsilon \in (0, 1)$,
\begin{equation}
\label{eq:func-pass-nisq}
P_{x \sim \pi_\theta}\!\left(\|U(x) - U_\text{target}\| \leq \varepsilon 
\;\middle|\; x \in \mathcal{M}_\text{struct}\right) 
\leq C_\theta \cdot e^{-\gamma_\text{NISQ}\, n},
\end{equation}
with $\gamma_\text{NISQ} = \Omega(d \log(1/\varepsilon))$.
\end{proposition}

\begin{proof}
We assume two regularity conditions on $f^\text{NISQ}$: 
\textbf{(R1)} \emph{generic full rank}, i.e., 
$J_{f^\text{NISQ}}(x) := \sqrt{\det(df^T df)} \geq j_0 > 0$ 
$\mu$-almost everywhere; and 
\textbf{(R2)} \emph{bounded fiber multiplicity}, 
$\#((f^\text{NISQ})^{-1}(y)) \leq N_\text{fiber}$ for 
$\mathcal{H}^D$-a.e.\ $y$ in the image. For ansatzes with bounded, 
non-redundant generators, $j_0 = \Omega(1)$ and 
$N_\text{fiber} \leq 2^{O(d)}$ by a Bezout-type bound on the algebraic 
degree of $f^\text{NISQ}$ as a trigonometric map. Pathological cases 
(e.g., barren-plateau ansatzes with $j_0 = e^{-\Omega(n)}$ 
\citep{mcclean2018barren}) violate (R1) and are addressed in 
Remark~(v).

If $U_\text{target} \notin \overline{f^\text{NISQ}(\mathcal{A}_d^\text{NISQ})}$, 
the preimage is empty and the bound is trivial. Otherwise, under (R1) 
the image is a $D$-dimensional immersed submanifold of $SU(2^n)$ with 
$D = \Theta(nd) \ll 4^n - 1$. By Federer's area formula 
\citep{federer1969geometric}, for any Borel 
$A \subseteq SU(2^n)$:
\begin{equation}
\label{eq:area-formula-nisq}
\int_{(f^\text{NISQ})^{-1}(A)} J_{f^\text{NISQ}}(x) \, d\mathrm{vol}(x)
\;=\; \int_{A \cap f^\text{NISQ}(\mathcal{A}_d^\text{NISQ})} 
\#\bigl((f^\text{NISQ})^{-1}(y)\bigr) \, d\mathcal{H}^D(y),
\end{equation}
where $\mathrm{vol}$ is Lebesgue measure on a product chart 
$\mathcal{A}_d^\text{NISQ} \subseteq [0, 2\pi]^G$ with $G = D = c_0 nd$, 
and $V_0 := \mathrm{vol}(\mathcal{A}_d^\text{NISQ}) = (2\pi)^G$. Setting 
$A = B_\varepsilon(U_\text{target})$, applying (R1) on the LHS and (R2) 
together with the standard submanifold volume bound 
$\mathcal{H}^D(B_\varepsilon \cap M) \leq \omega_D \varepsilon^D$ for 
$M$ of bounded extrinsic curvature \citep[\S 3.2.39]{federer1969geometric} 
on the RHS:
\begin{equation}
j_0 \cdot \mathrm{vol}\bigl((f^\text{NISQ})^{-1}(B_\varepsilon)\bigr) 
\;\leq\; N_\text{fiber} \cdot \omega_D \varepsilon^D,
\end{equation}
where $\omega_D = \pi^{D/2}/\Gamma(D/2+1) \leq O(1)$. Converting to the 
probability measure $\mu = \mathrm{vol}/V_0$:
\begin{equation}
\mu\bigl((f^\text{NISQ})^{-1}(B_\varepsilon(U_\text{target}))\bigr) 
\;\leq\; \frac{N_\text{fiber}}{j_0} \cdot 
\left(\frac{\omega_D^{1/D} \varepsilon}{2\pi}\right)^{D}.
\end{equation}
Substituting into Equation~\eqref{eq:pass-bound-nisq} and taking logs:
\begin{equation}
-\log P \;\geq\; c_0 nd \log\frac{2\pi}{\omega_D^{1/D}\varepsilon} 
- \log\frac{N_\text{fiber}}{j_0} - \log C_\theta.
\end{equation}
With $\log N_\text{fiber} = O(d)$, $\log(1/j_0) = O(1)$, and 
$\omega_D^{1/D} = O(1)$, the leading term dominates for 
$\varepsilon \in (0, 2\pi/e)$ and $n$ sufficiently large, giving 
$\gamma_\text{NISQ} = c_0 d \log(1/\varepsilon) - O(d/n) 
= \Omega(d \log(1/\varepsilon))$.
\end{proof}

\subsubsection{FTQC Analysis: Discrete Program Representation}
\label{app:func-sparsity-ftqc}

\textbf{Setup.} At the logical level, FTQC programs are naturally 
discrete: they consist of finite sequences over a fixed instruction set 
$\mathcal{I}$, which may include elementary fault-tolerant gates 
(Clifford+T), module invocations (QFT, modular arithmetic, amplitude 
estimation), lattice-surgery operations at the substrate level, and 
control-flow primitives. Each instruction has discrete arguments: qubit 
or register indices, module parameters, scheduling annotations. The 
choice of representation (textual IR such as QASM, AST-level 
representations, module-based descriptions as in 
Section~\ref{subsec:levels}, or planar quantum ISAs 
\citep{beverland2022assessing}) is immaterial for the argument; what 
matters is that the program space is a discrete set of finite description 
length.

Let $\mathcal{A}_d^\text{FTQC}$ denote the set of structurally valid 
program descriptions of length at most $G = \Theta(nd)$ over $\mathcal{I}$. 
Each position selects one of $|\mathcal{I}|$ instructions and up to 
$\mathcal{O}(1)$ qubit indices from $[n]$, giving
\begin{equation}
\label{eq:ftqc-cardinality}
|\mathcal{A}_d^\text{FTQC}| \;\leq\; (|\mathcal{I}| \cdot n^{\mathcal{O}(1)})^G 
\;=\; e^{\Theta(nd \log n)}.
\end{equation}
Compilation defines a map $f^\text{FTQC}: \mathcal{A}_d^\text{FTQC} \to SU(2^n)$. 
We bound the conditional pass rate by counting against the uniform 
distribution on $\mathcal{A}_d^\text{FTQC}$:
\begin{equation}
\label{eq:pass-bound-ftqc}
P_{x \sim \pi_\theta}\!\left(\|U(x) - U_\text{target}\| \leq \varepsilon\right)
\leq C_\theta \cdot 
\frac{|\{x \in \mathcal{A}_d^\text{FTQC} : \|U(x) - U_\text{target}\| \leq \varepsilon\}|}
     {|\mathcal{A}_d^\text{FTQC}|},
\end{equation}
where $C_\theta := \|\pi_\theta\|_\infty \cdot |\mathcal{A}_d^\text{FTQC}|$ 
measures deviation from uniform.

\begin{proposition}[Semantic entropy gap, FTQC]
\label{prop:func-sparsity-ftqc}
Under the setup above, for any $U_\text{target} \in SU(2^n)$ admitting a 
polynomial-length canonical description and any $\varepsilon \in (0, 1)$,
\begin{equation}
\label{eq:func-pass-ftqc}
P_{x \sim \pi_\theta}\!\left(\|U(x) - U_\text{target}\| \leq \varepsilon 
\;\middle|\; x \in \mathcal{M}_\text{struct}\right) 
\leq C_\theta \cdot e^{-\gamma_\text{FTQC}\, n},
\end{equation}
with $\gamma_\text{FTQC} = \Omega(d \log n)$.
\end{proposition}

\begin{proof}
By the Solovay--Kitaev theorem \citep{dawson2005solovay} applied to a 
universal fault-tolerant gate set, $U_\text{target}$ admits an 
$\varepsilon$-approximating program of length 
$G^* = \mathcal{O}(\mathrm{poly}(n, \log(1/\varepsilon)))$. The number of 
length-$G$ programs that $\varepsilon$-approximate $U_\text{target}$ is 
bounded by the product of (i) the number of $\varepsilon$-equivalent 
canonical forms of $U_\text{target}$, which is $\mathrm{poly}(n, \log(1/\varepsilon))$ 
for polynomially-describable targets, and (ii) the number of 
no-op insertions and rewrites of length $G - G^*$, which is at most 
$|\mathcal{I}|^{G - G^*} \cdot \mathrm{poly}(n)$. The numerator of 
Equation~\eqref{eq:pass-bound-ftqc} therefore satisfies
\begin{equation}
|\{x : \|U(x) - U_\text{target}\| \leq \varepsilon\}| 
\leq \mathrm{poly}(n, \log(1/\varepsilon)) \cdot |\mathcal{I}|^{G - G^*}.
\end{equation}
Dividing by Equation~\eqref{eq:ftqc-cardinality}, the qubit-index factor 
$n^{\Theta(G)}$ in the denominator dominates:
\begin{equation}
P(\cdot) \leq C_\theta \cdot 
\frac{\mathrm{poly}(n, \log(1/\varepsilon))}{n^{\Theta(G)}}
= C_\theta \cdot e^{-\Theta(G \log n) + \mathcal{O}(\log n)},
\end{equation}
giving $\gamma_\text{FTQC} = \Omega(d \log n)$ after substituting 
$G = \Theta(nd)$.
\end{proof}

\textbf{Remarks (Jointly for both NISQ and FTQC).} 

\textbf{(i) Where $C_\theta$'s boundedness comes from.} The bound holds 
pointwise for any $U_\text{target}$ in both regimes, but the constant 
$C_\theta$ distinguishes interesting from uninteresting cases. For 
targets in (or close to) the model's training distribution --- canonical 
algorithmic primitives such as QFT, modular arithmetic, or Grover 
diffusion operators --- the model may concentrate probability on the 
correct region of $\mathcal{A}_d$, making $C_\theta$ large enough to 
cancel the $e^{-\gamma n}$ factor. This is consistent with the empirical 
success of code-completion systems on textbook quantum algorithms 
\citep{dupuis2024qiskit, vishwakarma2024qiskit}: the model is 
not violating the proposition, it is operating in a regime where 
$C_\theta$ encodes implicit memorization. The proposition bites for 
\emph{out-of-distribution} targets, where $C_\theta$ is bounded by 
definition of generalization. The validity gap is therefore not about 
whether quantum circuits can in principle implement $U_\text{target}$ 
(universal approximation guarantees this), but about whether a 
polynomial-time sampler can find the correct circuit without already 
having seen the answer.

\textbf{(ii) Regime independence of the qualitative bound.} Both 
$\gamma_\text{NISQ} = \Omega(d \log(1/\varepsilon))$ and 
$\gamma_\text{FTQC} = \Omega(d \log n)$ yield exponential decay in $n$ 
for any polynomial depth $d$. The exponents differ in their dependence 
on tolerance ($1/\varepsilon$ for continuous parameters, $n$ for 
discrete qubit-index selections) but coincide qualitatively: in both 
settings, the conditional pass rate is $e^{-\Omega(n)}$ for fixed depth 
and precision, regardless of the physical substrate.

\textbf{(iii) The exponent gets stronger with depth.} A deeper ansatz 
(larger $d$) gives a \emph{tighter} bound in both regimes. This is the 
counterintuitive consequence of dimension counting / enumeration: a 
richer ansatz contains more candidate circuits per dimension of 
$SU(2^n)$, so the fraction landing near any specific target is smaller, 
not larger. The bound is not in conflict with universal approximation 
theorems for quantum circuits, which establish that sufficiently deep 
circuits can approximate any unitary, the propositions concern the 
\emph{fraction} of such approximators among all candidates, not their 
existence.

\textbf{(iv) Regime independence of the validity gap.} The argument is 
independent of whether the qubits are physical (NISQ) or encoded (FTQC). 
The dimension parameter $n$ in Equations~\eqref{eq:func-pass-nisq} and 
\eqref{eq:func-pass-ftqc} refers to the logical qubit count exposed to 
the algorithm designer. Hardware progress changes the substrate but not 
the target Hilbert-space dimension, and the validity gap is a property 
of the target space, not the implementation.

\textbf{(v) NISQ pathologies.}
The assumption (R1) $j_0 = \Omega(1)$ in 
Proposition~\ref{prop:func-sparsity-nisq} excludes ansatzes in the 
\emph{barren-plateau} regime, where $j_0 = e^{-\Omega(n)}$ and the 
bound degrades. Barren plateaus are a well-known failure mode of 
variational quantum algorithms 
\citep{mcclean2018barren, cerezo2021cost}; they signal that the 
parameter space is exponentially poorly conditioned for gradient-based 
optimization, which independently rules out efficient synthesis. The 
proposition therefore captures the relevant regime: either the ansatz 
is well-conditioned and our bound applies, or the ansatz is in a 
barren plateau and synthesis is already known to fail for 
optimization-theoretic reasons.

\textbf{(vi) Generic vs.\ symmetric ansatzes.}
(R1)--(R2) hold for ``generic'' parameterized ansatzes but may fail 
for ansatzes with built-in symmetries (equivariant ansatzes 
preserving problem-specific group actions, quantum convolutional 
networks). For such structured ansatzes the proposition can still apply 
under restriction to a symmetry-quotient parameter space, but the 
constants and effective dimension change. We treat this as outside the 
scope of this paper: structured ansatzes are precisely where 
verifier-centric agents (Section~\ref{sec:verifier}) succeed by 
exploiting structure, and the validity gap is most acute for 
unstructured synthesis problems where (R1)--(R2) hold by default.

\section{Property-preserving Rewrite}\label{sec:rewrite}
{
We illustrate Level~1--3 reasoning by a single depth-reduction move for the Cuccaro ripple-carry adder \cite{Cuccaro2004ripple}.

\textbf{Step 1 (Level 1: program over modules).}
Write the $n$-bit ripple-carry adder in terms of black-box modules $\mathrm{MAJ}$ and $\mathrm{UMA}$ and we get the initial program $x_0$:
\[
\begin{aligned}\label{eq:adder_struct0}
    &\textsc{Adder}(A,B,c_0,z): \\
    &\qquad \textsc{MAJ}(A_0,B_0,c_0) & \text{Line 1} \\
    &\qquad\textbf{for }i=1\text{ to }n-1:\;\textsc{MAJ}(A_i,B_i,A_{i-1}) & \text{Line 2} \\
    &\qquad\textsc{CNOT}(A_{n-1},z) \\
    &\qquad\textbf{for }i=n-1\text{ to }1:\;\textsc{UMA}(A_i, B_i, A_{i-1}) \\
    &\qquad\textsc{UMA}(A_0,B_0,c_0)
\end{aligned}
\]
$\mathrm{MAJ}(x,y,z)=(yz\oplus xy\oplus xz,y\oplus x,z\oplus x)$ propagates the carry forward, and $\mathrm{UMA}(x,y,z)=(x\oplus yz,x\oplus y\oplus z\oplus yz, x\oplus y\oplus yz)$ uncomputes the carry while producing sum bits.

\textbf{Step 2 (Level 2: choose gate decompositions).}
The MAJ gate can be inlined substituted with the following decomposition:
\[
\begin{aligned}
    &\textsc{MAJ}(a,b,c): \\
    &\quad\textsc{CNOT}(a,b);\;\textsc{CNOT}(a,c);\;\textsc{Toffoli}(c,b,a). \\
\end{aligned}
\]
Substituting this into Line 1 and Line 2 of the adder structure $x_0$ yields a new program $x_1$:
\[
\begin{aligned}\label{eq:adder_struct1}
    &\textsc{Adder}(A,B,c_0,z): \\
    &\quad \textsc{CNOT}(A_0,B_0);\;&\leftarrow l_1\\
    &\quad \textsc{CNOT}(A_0,c_0);\;\textsc{Toffoli}(c_0,B_0,A_0). \\
    &\quad\textbf{for }i=1\text{ to }n-1: \\
    &\qquad \textsc{CNOT}(A_i,B_i);\;&\leftarrow l_2 \\
    &\qquad \textsc{CNOT}(A_i,A_{i-1});\;\textsc{Toffoli}(A_{i-1},B_i,A_i) \\
    &\quad\cdots
\end{aligned}
\]
where the remainder of $x_1$ marked with ``$\cdots$" is identical to $x_0$. 

\textbf{Step 3 (Level 3: expose and apply a parallelization rewrite).}
Observe in $x_1$ an opportunity for parallelization shown as two lines marked as $l_1$ and $l_2$ respectively. This invokes a property-preserving rewrite step $a_1=\textsc{Edit}[\rho,\alpha]$ with $\rho=\text{Parallelize}$ and $\alpha=(l_1,l_2)$. A concrete manifestation of such rewrite is shown in Figure \ref{fig:cuccaro_adder}c with the CNOT gate in the (red) dotted box. The program $x_2$ after the rewrite reads as the following:
\[
\begin{aligned}\label{eq:adder_struct2}
    &\textsc{Adder}(A,B,c_0,z): \\
    &\quad\textbf{for }i=0\text{ to }n-1:\textsc{CNOT}(A_i,B_i);\\
    &\quad \textsc{CNOT}(A_0,c_0);\;\textsc{Toffoli}(c_0,B_0,A_0). \\
    &\quad\textbf{for }i=1\text{ to }n-1: \\
    &\qquad \textsc{CNOT}(A_i,A_{i-1});\;\textsc{Toffoli}(A_{i-1},B_i,A_i) \\
    &\quad\cdots
\end{aligned}
\]
From observing $x_2$, one sees that the rewrite collapses the lines $l_1$ and $l_2$ into a single line (or layer), and therefore simplifies the program.
}

\end{document}